%% file: example_paper.tex
\documentclass{article}

\usepackage{microtype}
\usepackage{graphicx}
\usepackage{booktabs} 

\usepackage{subfig}
\usepackage{multirow}
\usepackage{amsmath}
\DeclareMathOperator*{\argmax}{arg\,max}

\usepackage{algorithm}
\usepackage{algpseudocode}  

\usepackage{xcolor}

\usepackage{tcolorbox}

\usepackage[accepted]{icml2025}

\usepackage{amsmath}
\usepackage{amssymb}
\usepackage{mathtools}
\usepackage{amsthm}
\usepackage{booktabs} 
\usepackage{tabularx} 
\usepackage{multirow}
\usepackage{booktabs, adjustbox, makecell}
\usepackage{comment}
\usepackage[capitalize,noabbrev]{cleveref}

\theoremstyle{plain}
\newtheorem{theorem}{Theorem}[section]
\newtheorem{proposition}[theorem]{Proposition}
\newtheorem{lemma}[theorem]{Lemma}
\newtheorem{corollary}[theorem]{Corollary}
\theoremstyle{definition}
\newtheorem{definition}[theorem]{Definition}

\theoremstyle{remark}

\newcommand{\rbnote}[1]{ {\textcolor{blue} { ***Rahul edit: #1}}}
\usepackage[textsize=tiny]{todonotes}

\usepackage{subfig}

\usepackage{titlesec}
\titlespacing*{\section}{0pt}{8pt}{4pt}
\titlespacing*{\subsection}{0pt}{6pt}{3pt}

\usepackage{enumitem}
\setlist{itemsep=2pt, parsep=2pt, topsep=4pt}

\icmltitlerunning{VOILA: Value-of-Information Guided Fidelity Selection for Cost-Aware Multimodal Question Answering}

\begin{document}

\twocolumn[
\icmltitle{VOILA: Value-of-Information Guided Fidelity Selection for Cost-Aware Multimodal Question Answering}



\icmlsetsymbol{equal}{*}


\begin{icmlauthorlist}
\icmlauthor{Rahul Atul Bhope}{uci}
\icmlauthor{K. R. Jayaram}{ibm}
\icmlauthor{Vinod Muthusamy}{ibm}
\icmlauthor{Ritesh Kumar}{ibm}
\icmlauthor{Vatche Isahagian}{ibm}
\icmlauthor{Nalini Venkatasubramanian}{uci}

\end{icmlauthorlist}

\icmlaffiliation{uci}{UC Irvine, ICS}
\icmlaffiliation{ibm}{IBM Research AI}

\icmlcorrespondingauthor{Rahul Atul Bhope}{rbhope@uci.edu}

\icmlkeywords{Machine Learning, ICML}

\vskip 0.3in
]




\renewcommand{\thefootnote}{\fnsymbol{footnote}}
\footnotetext[1]{UC Irvine, ICS}
\footnotetext[2]{IBM Research AI}
\renewcommand{\thefootnote}{\arabic{footnote}}

\input{abstract}

\input{introduction}

\input{insights}

\input{new_methodology}

\input{evaluation}

\input{results}

\input{related_works}

\input{conclusion}





\bibliography{example_paper}
\bibliographystyle{icml2025}

\newpage
\appendix
\onecolumn

\appendix

\section*{Appendices}
\addcontentsline{toc}{section}{Appendices}

\renewcommand{\thesubsection}{\Alph{section}.\arabic{subsection}}

\begin{center}
\textbf{\Large Table of Contents}
\end{center}

\vspace{1em}

\noindent\textbf{Appendix~\ref{sec:complete_results}: Complete Results} \dotfill \pageref{sec:complete_results}
\begin{itemize}[leftmargin=2em, label={}]
    \item \ref{sec:complete_tables} Complete Accuracy and Cost Tables \dotfill \pageref{sec:complete_tables}
    \item \ref{sec:fidelity_distributions} Fidelity Selection Distributions \dotfill \pageref{sec:fidelity_distributions}
\end{itemize}

\vspace{0.5em}

\noindent\textbf{Appendix~\ref{sec:datasets}: Dataset Details} \dotfill \pageref{sec:datasets}
\begin{itemize}[leftmargin=2em, label={}]
    \item \ref{sec:vqa_datasets} VQA-v2, GQA, and TextVQA \dotfill \pageref{sec:vqa_datasets}
    \item \ref{sec:floodnet} FloodNet \dotfill \pageref{sec:floodnet}
    \item \ref{sec:locomo} LoCoMo Question Generation \dotfill \pageref{sec:locomo}
\end{itemize}

\vspace{0.5em}

\noindent\textbf{Appendix \ref{sec:theory}: Theoretical Foundations} \dotfill \pageref{sec:theory}
\begin{itemize}[leftmargin=2em, label={}]
    \item \ref{sec:bayes_optimal} Bayes-Optimal Fidelity Selector \dotfill \pageref{sec:bayes_optimal}
    \item \ref{sec:calibration} Learning Utility via Calibration \dotfill \pageref{sec:calibration}
    \item \ref{sec:optimality} Optimality Under Perfect Calibration \dotfill \pageref{sec:optimality}
    \item \ref{sec:regret} Regret with Imperfect Calibration \dotfill \pageref{sec:regret}
    \item \ref{sec:generalization} Generalization to Unseen Questions \dotfill \pageref{sec:generalization}
    \item \ref{sec:greedy_policy} Sequential Greedy Policy \dotfill \pageref{sec:greedy_policy}
    \item \ref{sec:implications} Implications for Adaptive Fidelity Selection \dotfill \pageref{sec:implications}
    \item \ref{sec:theory_empirics} Connection to Empirical Observations \dotfill \pageref{sec:theory_empirics}
\end{itemize}

\vspace{0.5em}

\noindent\textbf{Appendix D: Cost Model Details} \dotfill \pageref{sec:cost_model}
\begin{itemize}[leftmargin=2em, label={}]
    \item \ref{sec:fidelity_set} Fidelity Set \dotfill \pageref{sec:fidelity_set}
    \item \ref{sec:bandwidth} Empirical Bandwidth Proxy \dotfill \pageref{sec:bandwidth}
    \item \ref{sec:tier_cost} Tier-Aware Acquisition Cost \dotfill \pageref{sec:tier_cost}
    \item \ref{sec:normalization} Normalization \dotfill \pageref{sec:normalization}
    \item \ref{sec:deployment} Deployment Profiles \dotfill \pageref{sec:deployment}
    \item \ref{sec:measurement} Empirical Measurement Procedure \dotfill \pageref{sec:measurement}
    \item \ref{sec:sensitivity} Sensitivity Analysis \dotfill \pageref{sec:sensitivity}
\end{itemize}

\vspace{0.5em}

\noindent\textbf{Appendix E: Experimental Details} \dotfill \pageref{sec:experimental}
\begin{itemize}[leftmargin=2em, label={}]
    \item \ref{sec:features} Question Feature Extraction \dotfill \pageref{sec:features}
    \item \ref{sec:hyperparams} Calibration Model and Hyperparameters \dotfill \pageref{sec:hyperparams}
    \item \ref{sec:cross_validation} Cross-Validation Protocol \dotfill \pageref{sec:cross_validation}
    \item \ref{sec:voi_policy} VOI Policy and Threshold Selection \dotfill \pageref{sec:voi_policy}
    \item \ref{sec:hyperparam_search} Hyperparameter Search \dotfill \pageref{sec:hyperparam_search}
    \item \ref{sec:variance} Variance and Reporting \dotfill \pageref{sec:variance}
    \item \ref{sec:computational_overhead} Computational Overhead \dotfill \pageref{sec:computational_overhead}
\end{itemize}

\vspace{0.5em}

\noindent\textbf{Appendix F: Extended Ablation Studies} \dotfill \pageref{sec:ablations}
\begin{itemize}[leftmargin=2em, label={}]
    \item \ref{sec:ablation_design} Experimental Design \dotfill \pageref{sec:ablation_design}
    \item \ref{sec:ablation_results} Comprehensive Results \dotfill \pageref{sec:ablation_results}
    \item \ref{sec:voi_dominance} Why VOI Routing Dominates \dotfill \pageref{sec:voi_dominance}
    \item \ref{sec:predictor_robustness} Robustness to Predictor Architecture \dotfill \pageref{sec:predictor_robustness}
    \item \ref{sec:takeaways} Key Takeaways and Design Implications \dotfill \pageref{sec:takeaways}
\end{itemize}

\vspace{2em}

\clearpage

\input{complete_results}

\input{datasets}

\input{theory}

\input{cost_model}

\input{reproducibility}

\input{ablations_appendix}


\end{document}

%% file: abstract.tex
\begin{abstract}

Despite significant costs from retrieving and processing high-fidelity visual inputs, most multimodal vision-language systems operate at fixed fidelity levels. We introduce VOILA, a framework for Value-Of-Information-driven adaptive fidelity selection in Visual Question Answering (VQA) that optimizes \emph{what information to retrieve} before model execution. Given a query, VOILA uses a two-stage pipeline: a gradient-boosted regressor estimates correctness likelihood at each fidelity from question features alone, then an isotonic calibrator refines these probabilities for reliable decision-making. The system selects the minimum-cost fidelity maximizing expected utility given predicted accuracy and retrieval costs. We evaluate VOILA across three deployment scenarios using five datasets (VQA-v2, GQA, TextVQA, LoCoMo, FloodNet) and six Vision-Language Models (VLMs) with 7B--235B parameters. VOILA consistently achieves 50--60\% cost reductions while retaining 90--95\% of full-resolution accuracy across diverse query types and model architectures, demonstrating that pre-retrieval fidelity selection is vital to optimize multimodal inference under resource constraints.

\end{abstract}

%% file: introduction.tex
\section{Introduction}~\label{sec:intro}
Modern multimodal vision-language systems achieve remarkable accuracy on visual reasoning tasks by scaling models and training data~\cite{alayrac2022flamingo, li2023blip2, liu2024visual}, incurring substantially higher storage, bandwidth, and computational overheads. Recent work has addressed this by focusing on \emph{model selection} and \emph{adaptive routing} -- dynamically choosing between lightweight and heavyweight models based on query characteristics~\cite{chen2023frugalgpt, ding2024hybrid, wang2024routellm}. While reducing \emph{inference costs}, these approaches implicitly assume that input context is fixed and freely available.

In practice, this assumption rarely holds. Multimodal inputs often reside across heterogeneous storage tiers with varying retrieval latencies and costs. Mobile services like Google Photos and iCloud Photos store images at multiple resolutions: on-device caches contain compressed thumbnails and captions, while full-resolution originals reside in cloud storage~\cite{lonn2019smartphonepictureorganizationhierarchical, toderici2016variablerateimagecompression}. Similarly, enterprise systems use hot storage (e.g., AWS S3 Standard) for frequently accessed data and cold storage (e.g., AWS Glacier) for archival content, with retrieval latencies measured in minutes to hours~\cite{aws_s3_storage_classes}. Critically, these retrieval costs are incurred \emph{before} any model processes the input, making post-hoc model routing insufficient for end-to-end efficiency.


\begin{figure}[tb]
    \centering
    \includegraphics[width=0.9\linewidth]{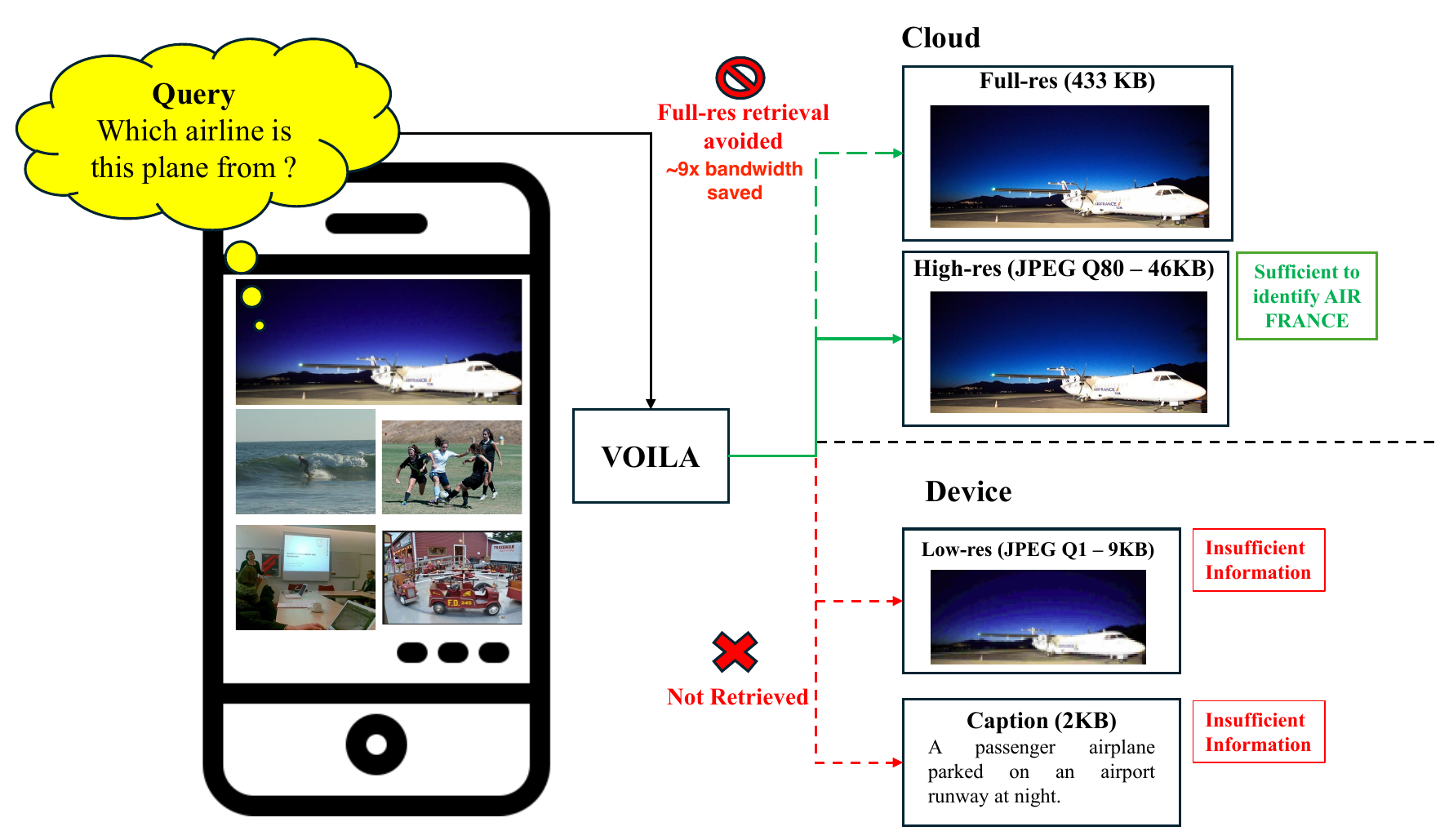}
    \caption{Given a user query, VOILA predicts the minimum visual fidelity required \emph{before retrieval}}.
    \label{fig:obs1}
\end{figure}

\begin{figure*}[h]
    \centering
    \subfloat[Query-dependent accuracy across input fidelities.\label{fig:obs1_dist}]{%
        \includegraphics[width=0.32\linewidth]{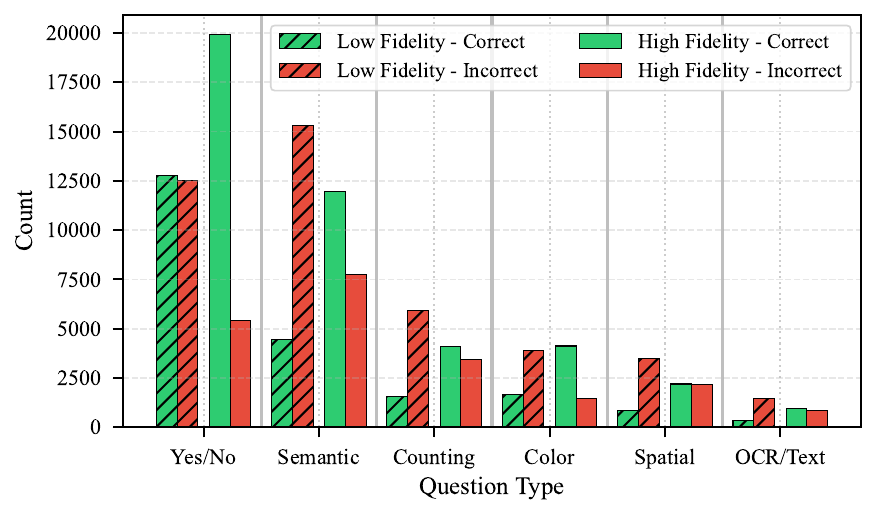}
    }
    \hfill
    \subfloat[Model confidence does not reliably indicate context sufficiency.\label{fig:obs2_confidence}]{%
        \includegraphics[width=0.32\linewidth]{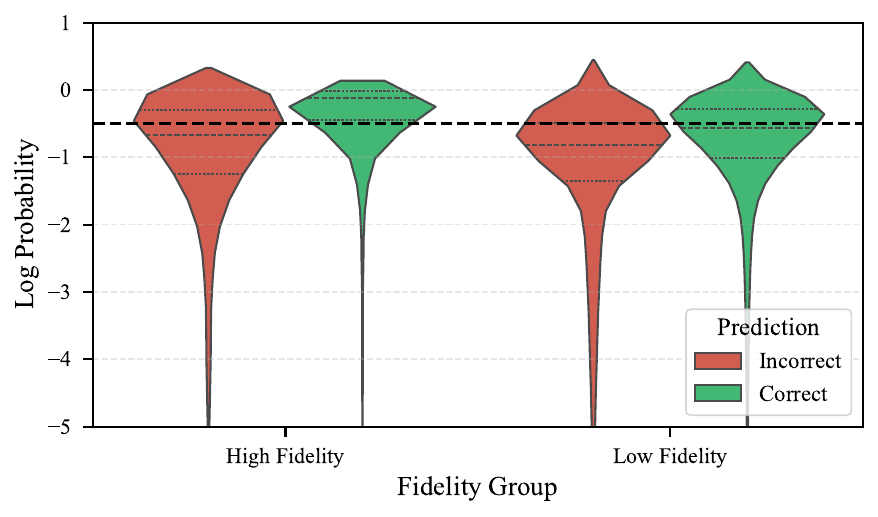}
    }
    \hfill
    \subfloat[Effect of model capacity under fixed input fidelity.\label{fig:model_scaling_fidelity}]{%
        \includegraphics[width=0.32\linewidth]{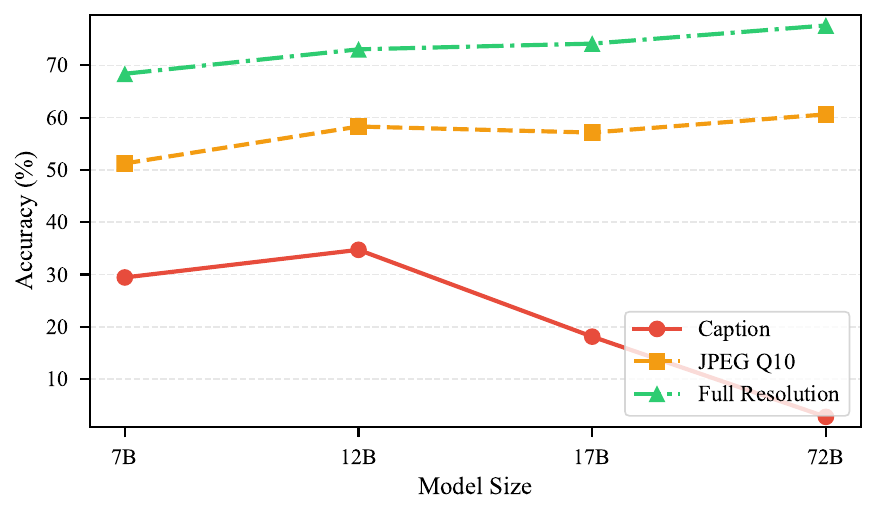}
    }
    \caption{Empirical evidence with Pixtral-12B on VQA-v2 motivating VOILA's design}~\label{fig:observations_combined}
\end{figure*}


We argue that \textbf{context selection at the pre-retrieval stage is a distinct and fundamental learning problem}, orthogonal to model routing. Before asking \emph{which model should answer this query?}, we ask: \emph{Which representation of the available context should be retrieved at all}? Figure~\ref{fig:obs1} shows that for ``Which airline is this plane from?'', a VLM cannot answer from a caption (2 KB) but identifies ``Air France'' from a compressed JPEG (46 KB), avoiding full-resolution retrieval (433 KB)—a 9$\times$ reduction. \textbf{Different queries need different fidelities}, predictable before retrieval. 


We formalize this setting as a \emph{cost-sensitive information acquisition problem} in \textbf{VOILA} (\underline{V}alue-\underline{O}f-\underline{I}nformation guided \underline{L}atency-\underline{A}ware fidelity selection), a principled framework for adaptive fidelity selection in Visual Question Answering (VQA). Given a query, VOILA employs a two-stage pipeline: (i) a gradient-boosted regressor (GBR) estimates, from question-only features, the probability that a given fidelity yields a correct answer and (ii) an isotonic regression then calibrates these estimates to produce reliable correctness probabilities. Using these calibrated predictions, VOILA selects the fidelity maximizing \emph{expected utility}---anticipated accuracy improvement minus acquisition cost---yielding a one-step value-of-information decision rule whose utility regret is bounded by calibration error as established in Appendix~\ref{sec:theory}.


Unlike adaptive inference and model routing (\S\ref{sec:related}), which optimize \emph{how to compute}, VOILA optimizes \emph{what information to acquire}. Pre-retrieval fidelity optimization is critical in three deployment scenarios: 

\begin{itemize}[nosep]
    \item \textbf{Edge-Cloud VQA Systems,} as in the example above (Figure~\ref{fig:obs1}). Network latency and bandwidth costs dominate total query latency.
    
    
    \item \textbf{Agentic Memory Systems.} Long-context conversational agents~\cite{wang2024locomo, park2023generative, wu2025longmemevalbenchmarkingchatassistants} store episodic memories across storage tiers: recent interactions in hot storage, high fidelity archival memories in cold storage, with high retrieval latencies.
    \item \textbf{Mission-Critical Cyber-Physical Systems.} In disaster response scenarios~\cite{rahnemoonfar2021floodnet, fangqidrone2023, fangqidrone2022, fangqidrone2021}, autonomous drones must decide which image fidelities to transmit over bandwidth-constrained networks, balancing latency against decision accuracy.
\end{itemize}

To our knowledge, VOILA is the first to address pre-retrieval fidelity selection for multimodal AI through value-of-information optimization. This paper makes the following technical contributions:

\begin{itemize}[nosep]
    \item \textbf{Empirical characterization:} We characterize systematic patterns in how query characteristics determine minimum information requirements, and show that existing heuristics (confidence-based escalation, model scaling) fail to address these patterns (\S\ref{sec:background}).

    \item \textbf{Value-of-information framework:} We introduce VOILA, a lightweight two-stage system combining gradient-boosted prediction with isotonic calibration to enable query-adaptive fidelity selection with bounded utility regret (\S\ref{sec:method}).

    \item \textbf{Problem formulation:} We formalize \textbf{pre-retrieval context selection} as a cost-sensitive information acquisition problem (\S\ref{sec:method}) and provide a theoretical analysis showing that VOILA implements a near Bayes-optimal policy when probability estimates are calibrated, with regret bounded by calibration quality (Appendix~\ref{sec:theory}).
    
    \item \textbf{Comprehensive evaluation} using \textbf{six VLMs with 7B -- 235B} parameters and \textbf{five datasets} from the three deployment scenarios mentioned above demonstrates \textbf{50--60\% cost reductions while retaining 90--95\% of full-resolution accuracy}, establishing pre-retrieval selection as a robust learning primitive (\S\ref{sec:experiments} and \S\ref{sec:results}, with extended ablations in Appendix~\ref{sec:ablations}).

\end{itemize}

%% file: insights.tex
\section{Why Pre-Retrieval Selection Matters?}~\label{sec:background}~\label{sec:emp_insights}
To motivate VOILA's design, we present three empirical observations, based on VQA-v2 dataset on Pixtral-12B VLM, examining variations in model behavior across input fidelities (text-only captions to full-resolution images).

\textbf{Observation 1: Information Requirements Vary Systematically Across Queries.} 
Figure~\ref{fig:obs1} in the previous section showed that different questions over the same image require different minimum input fidelities. Figure~\ref{fig:obs1_dist} quantifies this at scale: low-fidelity representations (captions, thumbnails) succeed on semantically coarse queries (yes/no, some object identification) but fail disproportionately on counting, color, spatial, and OCR tasks requiring fine-grained visual evidence.
In contrast, high-fidelity inputs substantially improve accuracy on these visually demanding categories while offering limited gains on queries already solvable at low fidelity. No single fidelity dominates across query types, motivating query-adaptive fidelity selection.

\textbf{Observation 2: Model Confidence Is an Unreliable Signal for Context Sufficiency.} 
Common adaptive systems retrieve low-cost context first and escalate to higher-fidelity inputs only when the model exhibits low confidence (token-level log-probabilities). However, model confidence is frequently misaligned with correctness when retrieved context lacks critical visual information. Figure~\ref{fig:obs2_confidence} shows that incorrect answers from low-fidelity inputs often receive log-probability scores comparable to correct answers, exceeding confidence thresholds that would prevent escalation and causing confident but incorrect predictions without accessing higher-fidelity representations.  Post-retrieval confidence is thus an unreliable indicator of context sufficiency, motivating pre-retrieval decision mechanisms.



\textbf{Observation 3: Model Scaling Does Not Compensate for Insufficient Information.}
Often multimodal query failures are addressed by  routing to larger VLMs. But, increasing model capacity does not resolve errors from insufficient input representations. Figure~\ref{fig:model_scaling_fidelity} shows that under fixed input fidelity, scaling yields limited and non-monotonic gains: caption-only inputs exhibit early saturation and degradation at larger sizes, JPEG Q10 shows marginal improvements, while full-resolution inputs consistently benefit. This indicates that model failures stem from missing visual evidence rather than inadequate reasoning, suggesting information sufficiency is a prerequisite for effective scaling, and that adaptive systems should prioritize decisions about what information to retrieve before selecting which model to apply.

%% file: new_methodology.tex
\section{VOILA: Overview}~\label{sec:method}~\label{sec:problem_setup}

\subsection{Problem Formulation} Let $(q, x)$ denote a visual question $q$ paired with an image $x$.
We assume access to a discrete, ordered set of input fidelities $\mathcal{F} = \{f_1, f_2, \dots, f_K\}$,
sorted by increasing acquisition cost, where each fidelity $f \in \mathcal{F}$ incurs a cost $c(f) > 0$ capturing retrieval latency, bandwidth, and compute overhead.
Lower fidelities correspond to inexpensive representations (e.g., captions or thumbnails), while higher fidelities provide richer visual information (e.g., full-resolution images).


For fidelity $f$, if the vision--language model (VLM) produces an answer $\hat{a}_f$, our objective is to select the fidelity that maximizes expected utility:
\[
f^\star
=
\arg\max_{f \in \mathcal{F}}
\underbrace{\Pr(\hat{a}_f \text{ is correct} \mid q)}_{\text{expected accuracy}} - \underbrace{\lambda c(f)}_{\text{acquisition cost}},
\]
where \textbf{the probability of correctness must be estimated \emph{without} retrieving fidelity $f$ or running the VLM at that fidelity}, and $\lambda$ controls the accuracy-cost trade-off. VOILA accomplishes this through question-conditioned learned estimators.


\subsection{Design Overview: A Three-Stage Pipeline}

VOILA adopts a three-stage pipeline that cleanly separates \emph{ranking} from \emph{probability estimation}, and operates independently for each fidelity.
: 

\begin{enumerate}[nosep]
    \item \textbf{Question-Conditioned Success Scoring} (\S\ref{sec:success_scoring}): Learn to predict, from question features alone, how likely each fidelity is to yield a correct answer.
    
    \item \textbf{Fidelity-Specific Probability Calibration} (\S\ref{sec:calibration}): Convert raw prediction scores into reliable probability estimates suitable for cost-aware decision-making.
    
    \item \textbf{Value-of-Information Selection} (\S\ref{sec:voi}): Use calibrated probabilities to select the fidelity that maximizes expected utility by balancing predicted accuracy gains against acquisition costs.
\end{enumerate}

\subsection{Question-Conditioned Success Scoring}~\label{sec:success_scoring}

Fig.~\ref{fig:observations_combined} illustrates that question features contain predictive signal about cross-fidelity performance \emph{before} any visual information is retrieved. To exploit this signal, we construct training data by executing the VLM at each fidelity $f$ on labeled examples and recording whether the produced answer is correct. Let $y_f \in \{0,1\}$ denote the correctness label for fidelity $f$. We extract a question feature representation $\phi(q)$ consisting of:
\begin{itemize}[nosep]
    \item TF--IDF features of the question text,
    \item lexical statistics (e.g., length, numeric indicators),
    \item coarse question-type (e.g., counting, color, spatial).
\end{itemize}


Using these features, we train a fidelity-specific regressor
\[
r_f(q) = g_f(\phi(q))
\]
to produce a raw score that correlates with the likelihood that the VLM will answer correctly at fidelity $f$. \textbf{Crucially, $g_f$ is a learned estimator of VLM behavior, not the VLM itself.} The regressor learns statistical patterns between question characteristics and cross-fidelity success rates during offline training. At inference time, \textbf{computing $r_f(q)$ requires only a forward pass through a lightweight tree ensemble---no visual retrieval, no VLM execution, no multimodal processing}. This is orders of magnitude faster and cheaper than actually running the VLM at each fidelity.

\begin{algorithm}[tb]
\caption{VOILA Inference Procedure}
\label{alg:voila}
\small
\begin{algorithmic}[1]
\State \textbf{Input:} Query $q$, fidelity set $\mathcal{F} = \{f_1, \ldots, f_K\}$, trade-off $\lambda$
\State \textbf{Output:} Selected fidelity $f^*$ and answer $\hat{a}_{f^*}$
\State \textcolor{blue}{// Phase 1: Fidelity selection (question-only, no VLM)}
\State Initialize $f \gets f_1$ \Comment{Start with lowest-cost fidelity}
\State Compute $\hat{p}_f(q)$ using calibrated predictor $\psi_f(r_f(q))$ 
\For{$f' \in \{f_2, \ldots, f_K\}$ in ascending cost order}
    \State Compute $\hat{p}_{f'}(q)$ using calibrated predictor $\psi_{f'}(r_{f'}(q))$ 
    \State $\mathrm{VOI} \gets [\hat{p}_{f'}(q) - \hat{p}_f(q)] - \lambda c(f')$
    \If{$\mathrm{VOI} > 0$}
        \State $f \gets f'$ \Comment{Escalate to higher fidelity}
        \State $\hat{p}_f(q) \gets \hat{p}_{f'}(q)$
    \Else
        \State \textbf{break} \Comment{No further escalation justified}
    \EndIf
\EndFor
\State \textcolor{blue}{// Phase 2: Execute VLM once at selected fidelity}
\State Retrieve visual representation at fidelity $f$ 
\State Execute VLM to produce answer $\hat{a}_f$ 
\State \Return $f, \hat{a}_f$
\end{algorithmic}
\end{algorithm}

We implement $g_f$ as a Gradient Boosting Regressor~\cite{prettenhofer2014gradient} for three practical reasons:
\begin{enumerate}[nosep]
    \item \textbf{Sparse, mixed features}: TF-IDF~\cite{salton1988term} embeddings combined with categorical indicators yield high-dimensional, sparse input spaces where tree-based methods excel.
    \item \textbf{Non-linear interactions}: GBR naturally captures complex dependencies (e.g., ``counting + small objects'' $\Rightarrow$ high-fidelity required) without manual feature engineering.
    \item \textbf{Robust default}: The ablation studies (\S\ref{sec:results}) confirm that while the specific regressor architecture matters less than calibration quality, GBR provides consistently strong performance across datasets.
\end{enumerate}

Importantly, the framework is agnostic to the specific regression model---any regressor that produces monotonic scores can be substituted.

\subsection{Fidelity-Specific Probability Calibration}
\label{sec:voila_calibration}

The regressor output $r_f(q)$ is \emph{not} a calibrated probability; rather, it is a \emph{question-conditioned success score} that orders queries by their relative likelihood of being answered correctly by the vision--language model at fidelity $f$.
Because $r_f(q)$ is not probabilistic, we apply fidelity-specific post-hoc calibration using isotonic regression~\cite{nuesch1991order}.
For each fidelity $f$, we fit an isotonic regression model
\[
\psi_f: r_f(q) \mapsto \hat{p}_f(q),
\]
on a held-out calibration split, where $\hat{p}_f(q) \in [0,1]$ estimates the probability that the vision--language model would answer question $q$ correctly if executed at fidelity $f$. By fitting a non-decreasing mapping from raw scores to empirical success frequencies on held-out data, isotonic regression enforces monotonicity by construction and corrects systematic over- and under-confidence in the raw scores.
This calibration step uses only $(r_f(q), y_f)$ pairs, where $y_f \in \{0,1\}$ denotes historical correctness at fidelity $f$. 
Isotonic regression offers three critical properties for our setting:
\begin{enumerate}
    \item \textbf{Monotonicity preservation}: Ensures that higher raw scores always map to higher probabilities, maintaining the rank ordering learned by the regressor.
    \item \textbf{Non-parametric}: Makes no assumptions about the functional form of miscalibration, allowing it to correct arbitrary systematic biases.
    \item \textbf{Proven calibration guarantees}: Isotonic regression is a standard tool in probability calibration with well-understood convergence properties (formal convergence guarantees in Appendix~\ref{sec:calibration}).
\end{enumerate}


\textbf{The key effect of the GBR + isotonic pipeline is \emph{probability separation}.}
Figure~\ref{fig:cali:scatter} contrasts raw confidence scores produced by the VLM (left) with the question-conditioned probabilities predicted by the GBR + isotonic pipeline (right).
The VLM’s own probabilities exhibit substantial overlap between correct and incorrect predictions, clustering in a narrow band and providing little discriminative signal when visual information is insufficient.
In contrast, the calibrated VOILA predictions are well separated and monotonic: queries that can be answered reliably  concentrate near high $\hat{p}_f(q)$ values, while queries that require additional visual detail are assigned lower probabilities.
This separation enables reliable value-of-information decisions, allowing VOILA to escalate to higher fidelities only when the expected accuracy gain justifies the additional acquisition cost.

\begin{figure}[h]
    \centering
    \includegraphics[width=\linewidth]{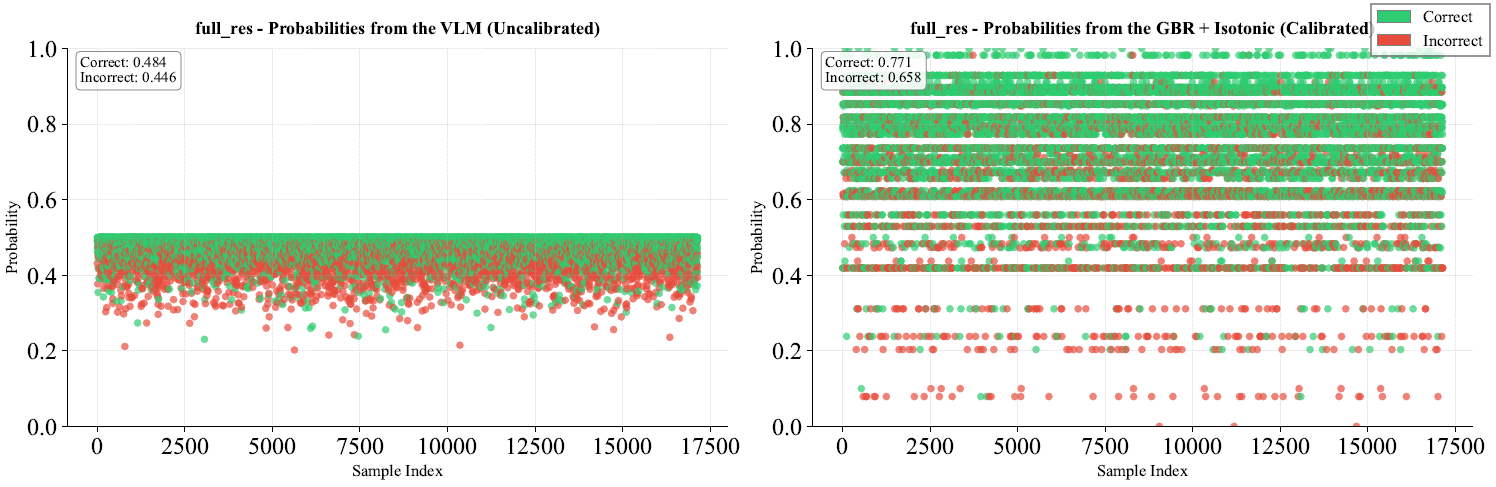}
    \caption{Calibrated vs. uncalibrated probabilities}

    \label{fig:cali:scatter}
\end{figure}

\subsection{Value-of-Information--Based Fidelity Selection}
\label{sec:voi}

With calibrated success probabilities $\hat{p}_f(q)$ in hand---\textbf{computed from question features alone, without any VLM execution}---we can now reason explicitly about the \emph{marginal value} of acquiring higher-cost fidelities. The key insight is that we should escalate to a more expensive representation only when the expected improvement in accuracy justifies the additional acquisition cost.

Starting from the lowest-cost fidelity $f_1$, VOILA evaluates whether escalating to the next higher fidelity $f_j$ is beneficial by computing the \textbf{Value of Information (VOI)}:
\[
\mathrm{VOI}(f_j \mid f_i)
=
\underbrace{[\hat{p}_{f_j}(q) - \hat{p}_{f_i}(q)]}_{\text{expected accuracy gain}} - \underbrace{\lambda c(f_j)}_{\text{acquisition cost}}.
\]

\textsc{VOILA} incrementally increases fidelity only when $\mathrm{VOI}(f_j \mid f_i) > \tau$, i.e., when the expected improvement in correctness outweighs the additional acquisition cost, where $\tau$ is a decision threshold determined during training and $\lambda$ denotes the cost penalty associated with escalation to a higher fidelity.
While a fully sequential decision process could plan over future escalations, our greedy policy offers three practical advantages:
\begin{enumerate}[nosep]
    \item \textbf{Computational efficiency}: Decisions require only a single forward pass through the calibrated predictors, avoiding expensive tree searchs.
    \item \textbf{Near-optimality}: Under accurate probability calibration, the greedy policy achieves bounded regret relative to the optimal sequential strategy (Appendix~~\ref{sec:regret}).
    \item \textbf{Interpretability}: The VOI criterion yields an explicit, human-interpretable decision rule that system operators can reason about and tune via $\lambda$ and $\tau$.
\end{enumerate}

Algorithm~\ref{alg:voila} summarizes the complete inference procedure. Crucially, VOILA is \textbf{model-agnostic} and operates entirely at inference time, requiring no modifications to the underlying VLM. The only offline cost is training the lightweight GBR predictors, isotonic calibrators and tuning $\lambda$, $\tau$, which takes minutes on standard hardware and generalizes across datasets (\S\ref{sec:results}). At inference time, the entire fidelity selection process operates on question text alone without ever touching visual data or executing expensive VLM inference until the optimal fidelity is already determined.

%% file: evaluation.tex
\section{Evaluation}~\label{sec:experiments}
Our evaluation aims to address three questions: (i) does \textsc{VOILA} consistently improve accuracy-cost trade-offs versus fixed-fidelity baselines? (ii) do gains persist across model scales and datasets? and (iii) which components are performance-critical?


\textbf{Models}: We evaluate six state-of-the-art vision--language models -- \emph{Pixtral-12B}~\cite{pixtral2024}, \emph{LLaMA-4-Maverick (17B$\times$12E)}~\cite{abdullah2025evolutionmetasllamamodels}, \emph{LLaVA-1.5-7B}~\cite{zhou2024tinyllava}, \emph{Qwen2-VL-72B}, \emph{Qwen2.5-VL-7B}, and \emph{Qwen-3-VL-235B-A22B}~\cite{bai2023qwen}. These models range from lightweight 7B parameter models to a 235B-parameter mixture-of-experts model, enabling us to study whether fidelity selection remains beneficial under aggressive model scaling. All models are used in a zero-shot or instruction-following setting.

\textbf{Datasets.}
We conduct experiments on five datasets spanning complementary query distributions and deployment regimes.
\textbf{VQA-v2}~\cite{goyal2017vqav2} and \textbf{GQA}~\cite{hudson2019gqa} serve as standard in-domain VQA benchmarks with diverse compositional and relational question types.
\textbf{TextVQA}~\cite{singh2019textvqa} emphasizes OCR-dependent and text-centric reasoning.
\textbf{FloodNet}~\cite{rahnemoonfar2021floodnet} models a mission-critical, bandwidth-constrained cyber--physical setting involving aerial imagery for disaster assessment.
Finally, \textbf{LoCoMo}~\cite{wang2024locomo} evaluates long-horizon agentic memory retrieval: starting from conversational memory items, we generate VQA-style question--answer pairs grounded in associated evidence text using GPT-4o.
Together, these datasets allow us to evaluate \textsc{VOILA} across varying degrees of visual complexity, linguistic structure, cost asymmetry, and out-of-distribution deployment constraints. Detailed descriptions of dataset preprocessing and question generation for 
LoCoMo are provided in 
Appendix~\ref{sec:datasets}.

\textbf{Input Fidelities.}
For each image, we construct a discrete, ordered set of input fidelities
\(
\mathcal{F}=\{\texttt{caption},\ \texttt{resize}_{32\times32},\ \texttt{jpeg\_q1},\ \texttt{jpeg\_q10},\ \texttt{full}\}
\),
chosen to reflect representations that are \emph{already materialized} in real-world multimodal systems rather than artificially synthesized intermediates.
Captions and low-resolution thumbnails are commonly cached on-device or in hot storage, while compressed JPEGs and full-resolution images reside in progressively slower and more expensive storage or transmission tiers.
These fidelities therefore correspond to natural breakpoints in edge--cloud VQA pipelines, agentic memory systems, and bandwidth-constrained cyber-physical deployments.
Fixed-fidelity baselines always operate at a single element of~$\mathcal{F}$, whereas \textsc{VOILA} adaptively selects the fidelity on a per-query basis.

\paragraph{Cost Model.}
Since VOILA optimizes \emph{pre-retrieval} fidelity selection decisions the cost function captures \emph{input acquisition overhead} rather than model-internal computation or token usage. Each fidelity \(f \in \mathcal{F}\) incurs an acquisition cost
\[
c(f) \;=\; \alpha\,B(f) + \beta\,L(f) + \gamma\,S(f),
\]
where \(B(f)\) is the empirically measured bandwidth ratio relative to the full-resolution image, \(L(f)\) captures retrieval latency associated with the storage or transmission tier, and \(S(f)\) accounts for storage access penalties.
Costs are normalized such that \(c(\texttt{full})=120\), and all other fidelities are scaled proportionally using their measured size ratios and tier-dependent penalties (details in Appendix~\ref{sec:cost_model}). VOILA depends only on \emph{relative} cost differences: uniform rescaling of \(c(f)\) does not affect routing behavior as long as higher fidelities remain more expensive, a property we verify empirically via sensitivity analysis. The complete cost formulation, tier-based pricing structure, and empirical 
measurement methodology are detailed in Appendix~\ref{sec:locomo}.

\textbf{Evaluation Metrics}: We report task accuracy and average cost per query for each method. Performance is assessed by comparing the accuracy--cost achieved by \textsc{VOILA} against fixed-fidelity baselines. For calibration quality, we additionally report Brier score in the ablation studies.

%% file: results.tex
\section{Results and Discussions}
\label{sec:results}
\begin{figure*}[htb]
    \centering
    \includegraphics[width=0.95\linewidth]{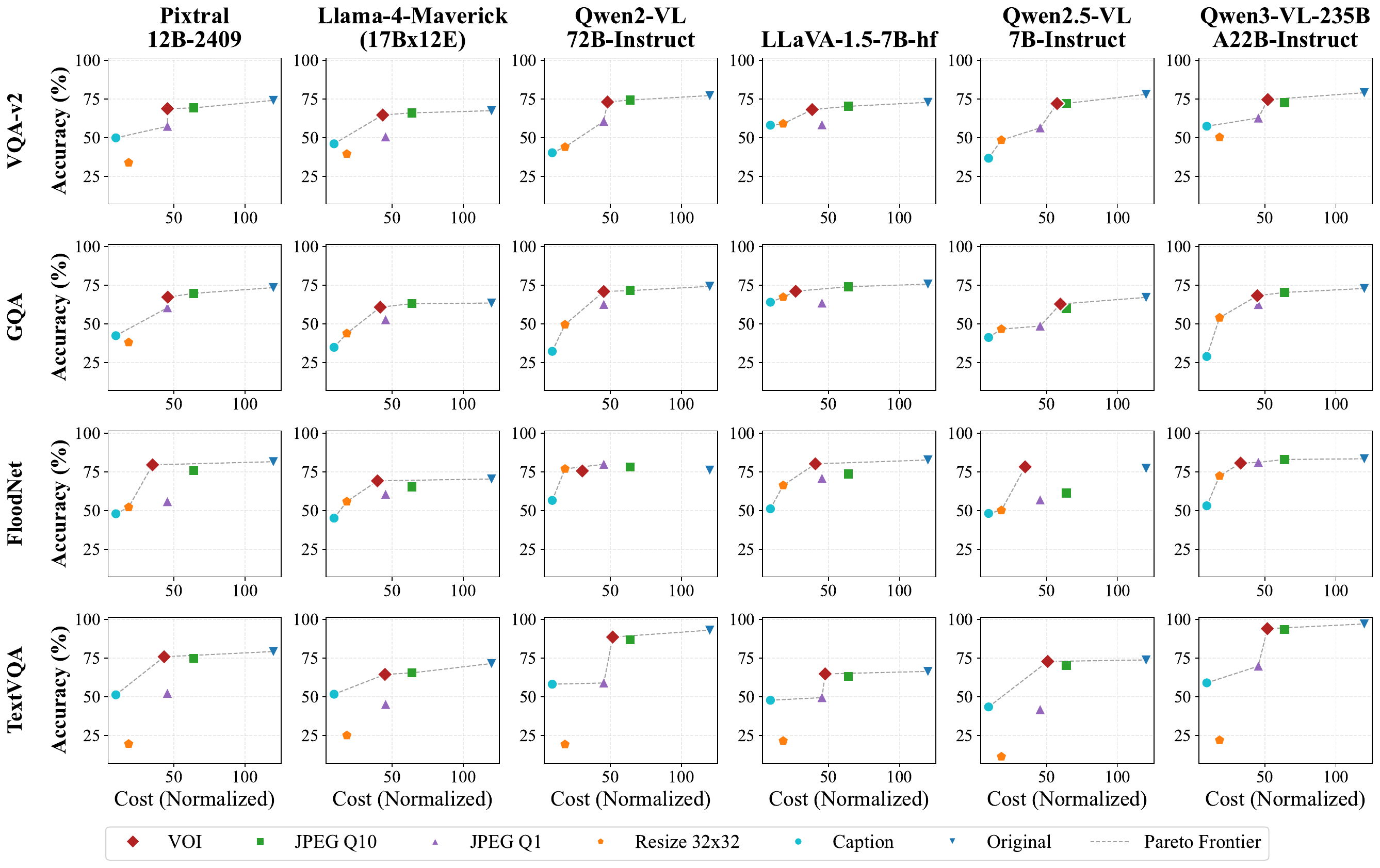}
    \caption{\textbf{Accuracy--cost Pareto frontiers across datasets and models.}
Each subplot plots accuracy versus normalized retrieval cost (upper left is better) for a dataset--model pair, comparing \textsc{VOILA} to fixed-fidelity baselines.
Dashed curves denote Pareto frontiers.
Across datasets and model scales, \textsc{VOILA} lies on or near the frontier, achieving near--full-resolution accuracy at 50--60\% lower average cost, while no single fixed fidelity is Pareto-optimal.}
    \label{fig:pareto}
\end{figure*}
\textbf{VOILA consistently matches near–full-resolution accuracy at dramatically lower cost}: The Pareto frontiers in Fig.~\ref{fig:pareto} show that VOILA lies consistently close to the full-resolution operating point across all four datasets and six models, while operating at substantially lower cost. The Complete accuracy/cost tables and fidelity distributions for all model-dataset combinations are 
provided in Appendix~\ref{sec:complete_results} (Table~\ref{tab:complete_results_vqa_gqa_flood_textvqa} and Figure~\ref{fig:voila_fidelity_distribution}). In most cases, VOILA achieves accuracy within 2–6 percentage points of always using full-resolution images, yet reduces normalized retrieval cost by roughly 50–60 \%. This pattern is stable across diverse visual regimes, including OCR-heavy TextVQA and bandwidth-constrained FloodNet, indicating that VOILA reliably identifies when high-fidelity visual information is genuinely required and avoids unnecessary retrieval otherwise. The resulting Pareto dominance demonstrates that VOILA effectively captures query-specific information requirements, delivering near-optimal accuracy without paying the full cost of maximal fidelity.


\textbf{Gains persist under distribution shift and scale}: The same qualitative Pareto improvements hold across model sizes ranging from 7B to 235B parameters and across datasets with distinct linguistic and visual distributions. Larger models improve absolute accuracy, but the relative gap between fixed-fidelity baselines remains, and VOILA continues to recover most of the accuracy of full-resolution inputs at a fraction of the cost. This indicates that model scaling alone cannot compensate for insufficient visual information: when critical details are missing, even very large models fail. VOILA’s consistent gains therefore reflect a complementary optimization axis—adaptive information acquisition—rather than an artifact of model capacity or dataset-specific tuning.

\begin{table}[htb]
\centering
\scriptsize
\setlength{\tabcolsep}{3pt}
\renewcommand{\arraystretch}{0.95}
\caption{\textbf{VOI out-of-distribution performance.}
Accuracy (\%) and avg. cost when training on one dataset and testing on the other.}
\label{tab:ood_pixtral_compact}

\resizebox{\columnwidth}{!}{
\begin{tabular}{lcc}
\toprule
\textbf{Model} 
& \textbf{VQA$\rightarrow$GQA} 
& \textbf{GQA$\rightarrow$VQA} \\
\midrule
Llama-4-Maverick (17Bx12E) 
& 58.57 / 48.55 
& 64.77 / 45.87 \\
Pixtral-12B-2409  
& 64.77 / 45.87 
& 63.34 / 42.26 \\
Qwen2.5-VL-7B     
& 62.87 / 40.71 
& 62.44 / 34.93 \\
LLaVA-1.5-7B     
& 69.12 / 64.60 
& 69.76 / 58.64 \\
Qwen3-VL-235B     
& 74.49 / 73.40 
& 75.98 / 67.98 \\
Qwen2-VL-72B     
& 65.61 / 34.23 
& 66.24 / 33.13 \\
\bottomrule
\end{tabular}}
\end{table}

\textbf{Out-of-Distribution Generalization Holds.}
Table~\ref{tab:ood_pixtral_compact} reports VOILA’s accuracy and average acquisition cost when trained on one dataset and evaluated on the other.
Across all six vision--language models, VOILA exhibits stable accuracy--cost behavior under both VQA$\rightarrow$GQA and GQA$\rightarrow$VQA transfer,
with only modest degradation in accuracy and limited variation in average cost. For example, Pixtral achieves 64.77\% accuracy at a cost of 45.87 when trained on VQA-v2 and tested on GQA, and 63.34\% accuracy at a cost of 42.26 in the reverse direction.
Similar symmetry is observed across model scales, from Qwen2.5-VL-7B (62.87\% / 40.71 vs.\ 62.44\% / 34.93) to the 235B-parameter Qwen3-VL model (74.49\% / 73.40 vs.\ 75.98\% / 67.98), indicating that VOILA’s routing policy does not overfit to a particular dataset distribution.
Notably, average acquisition costs remain close to the in-distribution VOI operating point, rather than collapsing to full-resolution retrieval under shift.
These results suggest that VOILA’s adaptive fidelity selection generalizes across datasets, and that the minimum visual information required to answer a query is primarily determined by query characteristics rather than dataset-specific statistics.

\begin{table}[htb]
\centering
\small
\caption{\textbf{LoCoMo memory recall (Acc \% / Cost).}
LoCoMo natively provides BLIP~\cite{li2023blip2} caption and full-image memory; intermediate fidelities are omitted for compactness but are used by VOI to reflect realistic agent memory systems.
}
\label{tab:locomo_ultracompact}
\setlength{\tabcolsep}{0pt}

\newcommand{\locomomin}[2]{%
\begin{tabular}{@{}p{0.26\columnwidth}p{0.20\columnwidth}@{}}
\toprule
\textbf{Method} & \textbf{Acc / Cost} \\
\midrule
Caption  & #1 \\
Original & #2 \\
\textbf{VOI} & \textbf{\locomoVOI} \\
\bottomrule
\end{tabular}%
}

\begin{tabular}{@{}p{0.49\columnwidth}@{\hfill}p{0.49\columnwidth}@{}}
\toprule

\textbf{Pixtral-12B-2409} &
\textbf{Qwen2.5-VL-7B-Instruct} \\
\addlinespace[1pt]

\begin{tabular}{@{}p{0.26\columnwidth}p{0.20\columnwidth}@{}}
\toprule
\textbf{Method} & \textbf{Acc / Cost} \\
\midrule
Caption  & 32.5 / 10 \\
Original & 46.7 / 120 \\
\textbf{VOI} & \textbf{43.6 / 69.5} \\
\end{tabular}
&
\begin{tabular}{@{}p{0.26\columnwidth}p{0.20\columnwidth}@{}}
\toprule
\textbf{Method} & \textbf{Acc / Cost} \\
\midrule
Caption  & 26.8 / 10 \\
Original & 73.0 / 120 \\
\textbf{VOI} & \textbf{67.1 / 52.3} \\
\end{tabular}
\\
\addlinespace[2pt]
\midrule
\addlinespace[2pt]

\textbf{LLaVA-1.5-7B-hf} &
\textbf{Qwen3-VL-235B-A22B} \\
\addlinespace[1pt]

\begin{tabular}{@{}p{0.26\columnwidth}p{0.20\columnwidth}@{}}
\toprule
\textbf{Method} & \textbf{Acc / Cost} \\
\midrule
Caption  & 43.7 / 10 \\
Original & 64.7 / 120 \\
\textbf{VOI} & \textbf{61.5 / 53.0} \\
\end{tabular}
&
\begin{tabular}{@{}p{0.26\columnwidth}p{0.20\columnwidth}@{}}
\toprule
\textbf{Method} & \textbf{Acc / Cost} \\
\midrule
Caption  & 27.4 / 10 \\
Original & 76.7 / 120 \\
\textbf{VOI} & \textbf{72.6 / 51.5} \\
\end{tabular}
\\
\addlinespace[2pt]
\midrule
\addlinespace[2pt]

\textbf{Llama-4-Maverick(17Bx12E)} &
\textbf{Qwen2-VL-72B-Instruct} \\
\addlinespace[1pt]

\begin{tabular}{@{}p{0.26\columnwidth}p{0.20\columnwidth}@{}}
\toprule
\textbf{Method} & \textbf{Acc / Cost} \\
\midrule
Caption  & 29.3 / 10 \\
Original & 60.8 / 120 \\
\textbf{VOI} & \textbf{56.2 / 50.9} \\
\end{tabular}
&
\begin{tabular}{@{}p{0.26\columnwidth}p{0.20\columnwidth}@{}}
\toprule
\textbf{Method} & \textbf{Acc / Cost} \\
\midrule
Caption  & 26.4 / 10 \\
Original & 76.0 / 120 \\
\textbf{VOI} & \textbf{68.1 / 40.8} \\
\end{tabular}
\\

\bottomrule
\end{tabular}
\end{table}

\begin{table*}[htb] \centering \scriptsize \renewcommand{\arraystretch}{0.95} \setlength{\tabcolsep}{2.5pt} \caption{\textbf{Ablations across 6 VLMs on VQA-v2 and GQA.} Entries report \textbf{Acc (\%) / Cost}, comparing VOI routing under different calibration methods and decision rules. Isotonic calibration (\textsc{VOILA}, bolded) consistently yields the best cost--accuracy tradeoff.
} \label{tab:ablation_combined} \resizebox{\textwidth}{!}{ \begin{tabular}{l|ccc|cc|ccc|cc} \toprule & \multicolumn{5}{c|}{\textbf{VQA-v2}} & \multicolumn{5}{c}{\textbf{GQA}} \\ \cmidrule(lr){2-6} \cmidrule(lr){7-11} \textbf{VLM} & \textbf{None} & \textbf{Iso} & \textbf{Temp} & \textbf{Acc} & \textbf{Thr} & \textbf{None} & \textbf{Iso} & \textbf{Temp} & \textbf{Acc} & \textbf{Thr} \\ \midrule Pixtral-12B & 66.24/52.18 & \textbf{68.67/45.41} & 66.89/56.32 & 74.06/120.0 & 69.21/63.90 & 64.12/51.23 & \textbf{67.37/45.66} & 63.84/49.67 & 73.52/120.0 & 69.81/63.90 \\ Qwen-3-VL-235B & 72.13/58.91 & \textbf{74.57/52.04} & 73.24/61.08 & 79.01/120.0 & 72.58/63.90 & 66.45/51.84 & \textbf{68.28/44.70} & 65.72/48.23 & 72.98/120.0 & 70.40/63.90 \\ LLaVA-1.5-7b & 65.84/44.73 & \textbf{68.07/38.52} & 64.23/46.91 & 72.80/120.0 & 70.21/63.90 & 68.92/32.45 & \textbf{71.18/26.86} & 67.18/35.62 & 75.77/120.0 & 74.11/63.90 \\ LLaMA-4 & 62.15/49.84 & \textbf{64.61/43.37} & 61.34/51.92 & 67.37/120.0 & 65.99/63.90 & 58.23/47.91 & \textbf{60.87/41.64} & 57.45/49.38 & 63.56/120.0 & 63.17/63.90 \\ Qwen-2-VL-72B & 70.48/54.72 & \textbf{73.01/48.00} & 71.15/58.34 & 77.16/120.0 & 74.27/63.90 & 68.34/51.28 & \textbf{71.03/45.37} & 67.92/53.71 & 74.35/120.0 & 71.57/63.90 \\ Qwen-2.5-VL-7B & 69.23/63.84 & \textbf{72.01/57.31} & 68.91/67.15 & 77.98/120.0 & 72.10/63.90 & 60.18/65.73 & \textbf{62.92/59.66} & 58.94/68.42 & 67.22/120.0 & 60.14/63.90 \\ \bottomrule \end{tabular}} \end{table*}

\begin{table}[htb] \centering \scriptsize \setlength{\tabcolsep}{3pt} \renewcommand{\arraystretch}{0.95} \caption{\textbf{Predictor robustness under fixed calibration (Isotonic), for Pixtral on VQA-v2 and GQA.} Performance is broadly similar across predictors once calibration is fixed, indicating VOILA is not brittle to predictor choice. Each entry reports \textbf{Acc (\%) / Cost}.} \label{tab:predictor_robustness_pixtral} \resizebox{\columnwidth}{!}{ \begin{tabular}{lcccc} \toprule \textbf{Dataset} & \textbf{GBR} & \textbf{LogReg} & \textbf{Ridge} & \textbf{MLP} \\ \midrule VQA-v2 & 68.67/45.41 & 67.92/47.23 & 68.14/46.58 & 67.35/48.91 \\ GQA & 67.37/45.66 & 67.81/46.84 & 67.73/45.12 & 68.05/49.27 \\ \bottomrule \end{tabular}} \end{table}

\textbf{VOILA Extends Beyond Static Benchmarks to Agentic Memory Systems.}
LoCoMo evaluates long-horizon agentic memory under a highly constrained setting in which only caption-level summaries are cached and full-image recall is expensive.
Despite being trained exclusively on static VQA benchmarks (VQA-v2 and GQA), VOI generalizes effectively to LoCoMo, consistently achieving strong accuracy--cost tradeoffs across models without any task- or dataset-specific tuning. For Qwen2.5-VL-7B, VOI attains 67.1\% accuracy at a cost of 52.3, substantially outperforming all cheaper fixed-fidelity baselines while avoiding the high cost of always recalling the original image (72.9\% at cost 120).
Similarly, for LLaVA-1.5-7B, VOI reaches 61.5\% accuracy at cost 53.0, closely matching full-image performance (64.7\%) at less than half the retrieval cost.
Across models, VOI consistently recovers a large fraction of full-image accuracy while operating in a regime where indiscriminate high-fidelity recall would be prohibitively expensive. Crucially, these results are obtained in an out-of-distribution setting: VOI is trained on question--answer pairs from VQA-v2 and GQA, yet deployed on LoCoMo without access to intermediate visual fidelities.
This demonstrates that VOI learns query-driven signals about when additional visual information is valuable, rather than overfitting to dataset-specific heuristics or fidelity availability.
Overall, VOI enables robust, cost-aware memory retrieval while using high-fidelity perception only when necessary.

\textbf{VOILA Enables Cost-Aware Perception in IoT / CPS Systems.} FloodNet exhibits a distinct accuracy--cost structure, underscoring the need for adaptive perception in cyber--physical settings. Caption-only inputs perform poorly (often $>25$ accuracy points lower), while fixed mid-quality JPEG baselines are consistently Pareto-dominated, incurring high transmission cost without commensurate gains. In contrast, VOILA selects the minimum sufficient fidelity per query, achieving near--full-resolution accuracy while reducing average retrieval cost by 45--70\%. These results indicate that in bandwidth-constrained edge--cloud and drone deployments, adaptive pre-retrieval fidelity selection is essential for reliable, decision-critical perception.

\textbf{Calibration is the central enabler of effective VOI routing across both VQA-v2 and GQA.}
Across all six VLMs and both datasets (Table~\ref{tab:ablation_combined}), isotonic calibration (VOILA default) consistently delivers the strongest cost–accuracy tradeoff, reducing acquisition cost by roughly 15–30 \% relative to uncalibrated routing while maintaining or improving accuracy. This pattern holds uniformly across model scales and datasets, confirming that calibration primarily improves \emph{routing efficiency} rather than raw task performance by enabling reliable confidence estimation for escalation decisions. In contrast, temperature scaling often increases cost without commensurate accuracy gains, and uncalibrated predictors exhibit unstable behavior due to systematic over- or under-confidence. These results demonstrate that accurate probability calibration is a prerequisite for translating success prediction into meaningful utility gains in sequential, cost-sensitive inference.

\textbf{VOI routing dominates heuristic policies while remaining robust to predictor choice, with GBR offering a strong practical default.} Across datasets, VOILA consistently lies on the Pareto frontier: accuracy-only routing incurs near full-resolution cost for marginal gains, while fixed-threshold heuristics trade cost savings for unstable and often significant accuracy loss. Predictor ablations (Table~\ref{tab:predictor_robustness_pixtral}) show that with isotonic calibration, lightweight models (GBR, logistic, ridge, MLP) achieve comparable accuracy–cost tradeoffs, indicating that routing quality is driven by probability calibration rather than predictor class. We adopt GBR as a strong practical default due to its ability to capture non-linear question–fidelity interactions with minimal tuning, while retaining flexibility to swap predictors without affecting VOILA’s core behavior.

\textbf{VOILA Overheads.} VOILA imposes minimal overhead: 0.45ms per query for fidelity selection versus 2.4--62.8 seconds for VLM inference. The system requires only 0.04MB memory and runs entirely on CPU, making it suitable for edge devices, smartphones, and IoT deployments. Training is fast (21s per configuration) and parallelizable, with hyperparameter search completing in under 10 minutes—a one-time offline cost that generalizes across query distributions. See Appendix~\ref{sec:computational_overhead} for detailed measurements.






%% file: related_works.tex
\section{Related Work}~\label{sec:related}

\noindent\textbf{Adaptive Inference and Model Routing.}
Prior work studies adaptive computation by routing queries across models or dynamically allocating compute, including early-exit and adaptive depth mechanisms \cite{zhou2020bert, schuster2022confident} and cost-aware LLM routing \cite{chen2023frugalgpt, shnitzer2023large, ding2024hybrid, wang2024routellm}.
These methods optimize \emph{how} inference is performed given fixed inputs, whereas VOILA focuses on \emph{what information} to retrieve before inference.

\noindent\textbf{Efficiency and Compression in Multimodal QA.}
Several approaches reduce multimodal cost by compressing or simplifying visual context, such as lightweight fusion \cite{farazi2021accuracy}, learned context compression \cite{weng2024learning}, and reliance on question-only or latent representations \cite{li2024context, wang2021latent}.
In contrast, VOILA operates \emph{before} retrieval, deciding whether higher-fidelity visual input is needed at all.

\noindent\textbf{Value-of-Information and Cost-Sensitive Acquisition.}
VOILA builds on classical value-of-information and budgeted acquisition frameworks \cite{bilgic2011value, karkkainen2019cost}, including active acquisition in multimodal or temporal settings \cite{kossen2022active} and cost-sensitive decision systems outside ML \cite{Wickens2001}.
Unlike prior work focused on tabular or sequential features, VOILA applies VOI to pre-retrieval multimodal perception.

\noindent\textbf{Systems Context for Multimodal Retrieval.}
Multimodal systems commonly store visual data across heterogeneous tiers, motivating retrieval-aware optimization.
Related work examines tiered storage and retrieval costs in cloud and mobile systems \cite{hort2021survey, zhang2024autonomous}, as well as agentic memory and long-horizon multimodal reasoning \cite{park2023generative, wang2024locomo}.
VOILA explicitly exploits this structure by selecting the minimum sufficient fidelity before accessing higher-cost tiers.

%% file: conclusion.tex
\section{Conclusion}

We introduced \textsc{VOILA}, a value-of-information framework for pre-retrieval fidelity selection in multimodal question answering. Using question-conditioned success prediction with fidelity-specific probabilistic calibration, VOILA selects the minimum-cost visual representation before retrieval.  Across five datasets and six vision--language models (7B--235B), VOILA retains 90--95 \% accuracy of full-resolution processing while reducing average retrieval cost by 50--60\%, with gains holding under model scaling, OOD data, and diverse deployment settings. These results, together with formal regret guarantees, establish adaptive pre-retrieval information acquisition as a principled and practical complement to post-retrieval model routing in multimodal systems.

%% file: complete_results.tex
\section{Complete Results}
\label{sec:complete_results}

\subsection{Complete Accuracy and Cost Tables}
\label{sec:complete_tables}

\begin{table*}[htbp]
\centering
\tiny
\caption{
\textbf{Complete Comparison Table (VQA-v2, GQA, FloodNet, TextVQA).}
Accuracy (\%) and average cost for each method across datasets and models.
All costs use the normalized acquisition model $c(f)$.
\textbf{VOI} rows are bolded.
}
\label{tab:complete_results_vqa_gqa_flood_textvqa}
\resizebox{\linewidth}{!}{
\begin{tabular}{llcccccccc}
\toprule
\multirow{2}{*}{\textbf{Model}} & \multirow{2}{*}{\textbf{Method}}
& \multicolumn{2}{c}{\textbf{VQA-v2}}
& \multicolumn{2}{c}{\textbf{GQA}}
& \multicolumn{2}{c}{\textbf{FloodNet}}
& \multicolumn{2}{c}{\textbf{TextVQA}} \\
\cmidrule(lr){3-4}\cmidrule(lr){5-6}\cmidrule(lr){7-8}\cmidrule(lr){9-10}
& & \textbf{Acc} & \textbf{Cost}
& \textbf{Acc} & \textbf{Cost}
& \textbf{Acc} & \textbf{Cost}
& \textbf{Acc} & \textbf{Cost} \\
\midrule

\multirow{6}{*}{Pixtral-12B-2409}
& Caption & 49.83 & 9.1 & 42.45 & 9.1 & 47.84 & 9.1 & 51.31 & 9.1 \\
& \textbf{VOI} & \textbf{68.67} & \textbf{45.41} & \textbf{67.37} & \textbf{45.66} & \textbf{79.48} & \textbf{34.94} & \textbf{75.92} & \textbf{43.09} \\
& JPEG Q10 & 69.21 & 63.9 & 69.81 & 63.9 & 75.75 & 63.9 & 74.82 & 63.9 \\
& JPEG Q1 & 57.27 & 45.4 & 60.53 & 45.4 & 55.70 & 45.4 & 52.28 & 45.4 \\
& Resize $32{\times}32$ & 33.80 & 18.1 & 38.10 & 18.1 & 52.05 & 18.1 & 19.57 & 18.1 \\
& Original & 74.06 & 120.0 & 73.52 & 120.0 & 81.51 & 120.0 & 79.29 & 120.0 \\
\midrule

\multirow{6}{*}{Llama-4-Maverick (17Bx12E)}
& Caption & 45.98 & 9.1 & 34.93 & 9.1 & 45.01 & 9.1 & 51.70 & 9.1 \\
& \textbf{VOI} & \textbf{64.61} & \textbf{43.37} & \textbf{60.87} & \textbf{41.64} & \textbf{69.14} & \textbf{39.68} & \textbf{64.52} & \textbf{44.98} \\
& JPEG Q10 & 65.99 & 63.9 & 63.17 & 63.9 & 65.23 & 63.9 & 65.48 & 63.9 \\
& JPEG Q1 & 50.45 & 45.4 & 52.76 & 45.4 & 60.41 & 45.4 & 45.11 & 45.4 \\
& Resize $32{\times}32$ & 39.42 & 18.1 & 43.92 & 18.1 & 55.73 & 18.1 & 25.11 & 18.1 \\
& Original & 67.37 & 120.0 & 63.56 & 120.0 & 70.35 & 120.0 & 71.56 & 120.0 \\
\midrule

\multirow{6}{*}{Qwen2-VL-72B-Instruct}
& Caption & 40.21 & 9.1 & 32.34 & 9.1 & 56.48 & 9.1 & 58.19 & 9.1 \\
& \textbf{VOI} & \textbf{73.01} & \textbf{48.00} & \textbf{71.03} & \textbf{45.37} & \textbf{75.51} & \textbf{30.30} & \textbf{88.65} & \textbf{51.64} \\
& JPEG Q10 & 74.27 & 63.9 & 71.57 & 63.9 & 78.07 & 63.9 & 86.84 & 63.9 \\
& JPEG Q1 & 60.49 & 45.4 & 62.68 & 45.4 & 79.84 & 45.4 & 58.96 & 45.4 \\
& Resize $32{\times}32$ & 43.90 & 18.1 & 49.63 & 18.1 & 76.85 & 18.1 & 19.27 & 18.1 \\
& Original & 77.16 & 120.0 & 74.35 & 120.0 & 76.08 & 120.0 & 93.11 & 120.0 \\
\midrule

\multirow{6}{*}{LLaVA-1.5-7B-hf}
& Caption & 58.01 & 9.1 & 64.09 & 9.1 & 51.05 & 9.1 & 47.80 & 9.1 \\
& \textbf{VOI} & \textbf{68.07} & \textbf{38.52} & \textbf{71.18} & \textbf{26.86} & \textbf{80.12} & \textbf{40.77} & \textbf{64.85} & \textbf{47.62} \\
& JPEG Q10 & 70.21 & 63.9 & 74.11 & 63.9 & 73.64 & 63.9 & 63.25 & 63.9 \\
& JPEG Q1 & 58.25 & 45.4 & 63.51 & 45.4 & 70.76 & 45.4 & 49.50 & 45.4 \\
& Resize $32{\times}32$ & 58.99 & 18.1 & 67.37 & 18.1 & 66.22 & 18.1 & 21.50 & 18.1 \\
& Original & 72.80 & 120.0 & 75.77 & 120.0 & 82.61 & 120.0 & 66.45 & 120.0 \\
\midrule

\multirow{6}{*}{Qwen2.5-VL-7B-Instruct}
& Caption & 36.65 & 9.1 & 41.28 & 9.1 & 48.06 & 9.1 & 43.48 & 9.1 \\
& \textbf{VOI} & \textbf{72.01} & \textbf{57.31} & \textbf{62.92} & \textbf{59.66} & \textbf{78.17} & \textbf{34.85} & \textbf{72.90} & \textbf{50.72} \\
& JPEG Q10 & 72.10 & 63.9 & 60.14 & 63.9 & 61.24 & 63.9 & 70.26 & 63.9 \\
& JPEG Q1 & 56.26 & 45.4 & 48.66 & 45.4 & 56.81 & 45.4 & 41.74 & 45.4 \\
& Resize $32{\times}32$ & 48.38 & 18.1 & 46.70 & 18.1 & 50.06 & 18.1 & 11.35 & 18.1 \\
& Original & 77.98 & 120.0 & 67.22 & 120.0 & 77.19 & 120.0 & 73.87 & 120.0 \\
\midrule

\multirow{6}{*}{Qwen3-VL-235B-A22B-Instruct}
& Caption & 57.36 & 9.1 & 28.97 & 9.1 & 52.99 & 9.1 & 59.05 & 9.1 \\
& \textbf{VOI} & \textbf{74.57} & \textbf{52.04} & \textbf{68.28} & \textbf{44.70} & \textbf{80.62} & \textbf{32.94} & \textbf{94.10} & \textbf{51.75} \\
& JPEG Q10 & 72.58 & 63.9 & 70.40 & 63.9 & 82.95 & 63.9 & 93.65 & 63.9 \\
& JPEG Q1 & 62.62 & 45.4 & 62.58 & 45.4 & 81.06 & 45.4 & 69.85 & 45.4 \\
& Resize $32{\times}32$ & 50.21 & 18.1 & 54.08 & 18.1 & 72.31 & 18.1 & 22.00 & 18.1 \\
& Original & 79.01 & 120.0 & 72.98 & 120.0 & 83.39 & 120.0 & 97.15 & 120.0 \\
\bottomrule
\end{tabular}
}
\end{table*}

\subsection{Fidelity Selection Distributions}
\label{sec:fidelity_distributions}

\begin{figure*}[t]
    \centering
    \includegraphics[width=\linewidth]{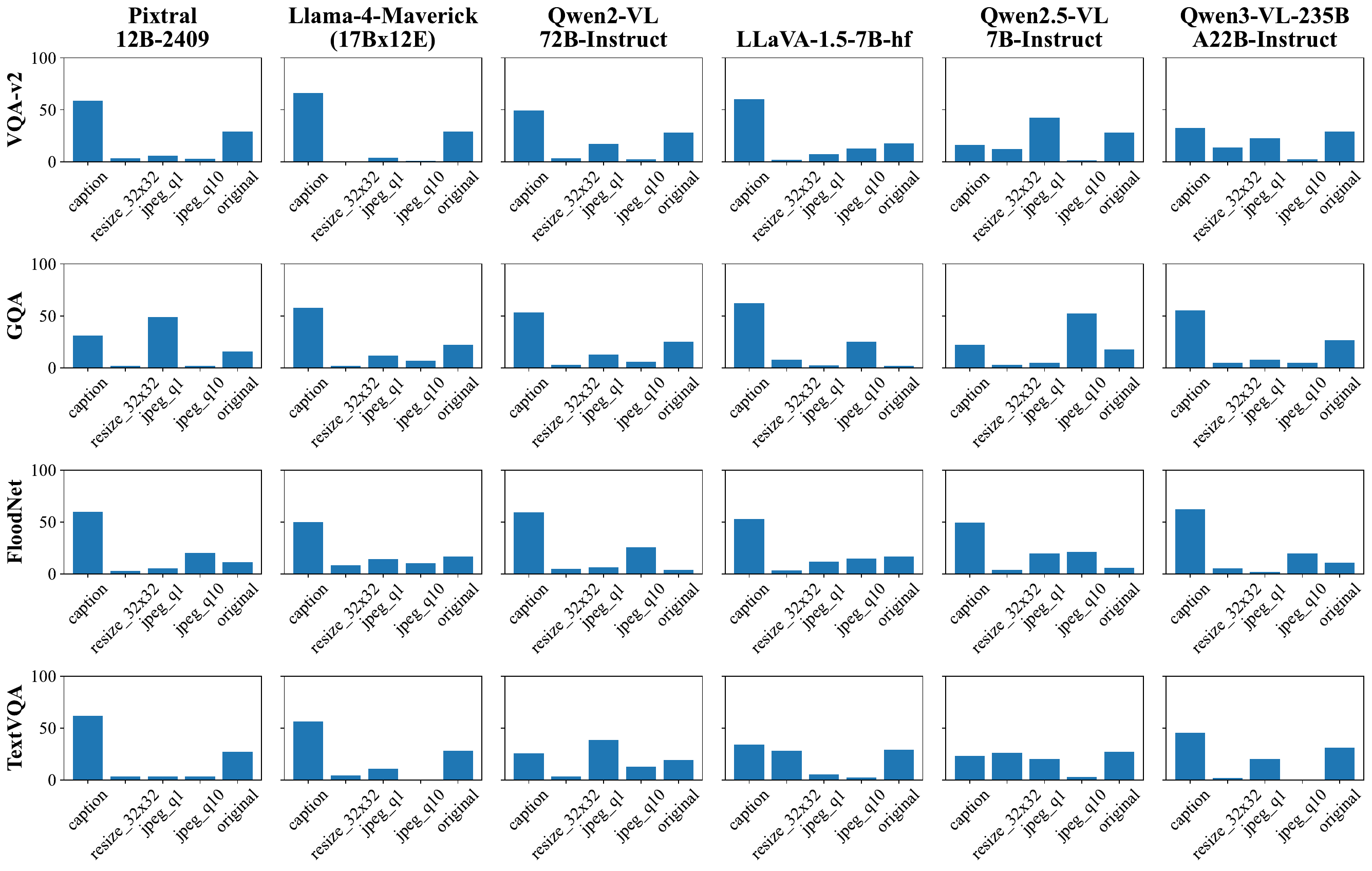}
    \caption{
    \textbf{Learned fidelity selection distributions under VOI routing.}
    Each subplot reports the percentage of queries routed to each input fidelity by \textsc{VOILA}, across four datasets (rows) and six vision--language models (columns).
    Low-cost representations (captions and low-resolution thumbnails) dominate across all settings, while higher-cost fidelities (JPEG Q1/Q10 and full-resolution images) are selected selectively when expected utility justifies their acquisition cost.
    }
    \label{fig:voila_fidelity_distribution}
\end{figure*}

\pagebreak

%% file: datasets.tex
\section{Appendix: Datasets}
\label{sec:datasets}

We evaluate \textsc{VOILA} on five datasets spanning standard VQA benchmarks, text-centric reasoning, cyber--physical perception, and long-horizon agentic memory.
Together, these datasets allow us to assess accuracy--cost tradeoffs under diverse visual content, linguistic structure, and deployment constraints.

\subsection{VQA-v2, GQA, and TextVQA}
\label{sec:vqa_datasets}
\paragraph{VQA-v2.}
VQA-v2 \cite{goyal2017vqav2} is a large-scale visual question answering benchmark consisting of real-world images paired with free-form questions and short answers.
It covers a broad range of question types, including object recognition, counting, attributes, and commonsense reasoning, and is designed to reduce language priors through balanced answer distributions.
We use the standard train/validation splits.

\paragraph{GQA.}
GQA \cite{hudson2019gqa} focuses on compositional and relational reasoning over images.
Questions are generated from structured scene graphs and emphasize multi-step reasoning, spatial relations, and logical composition.
We use GQA to test whether VOI generalizes beyond surface-level recognition to structured visual reasoning.

\paragraph{TextVQA.}
TextVQA \cite{singh2019textvqa} evaluates text-centric visual understanding, requiring models to read and reason over scene text embedded in images.
This dataset stresses OCR-dependent perception and complements VQA-v2 and GQA by emphasizing linguistic grounding in visual text.

\subsection{FloodNet.}
\label{sec:floodnet}
FloodNet \cite{rahnemoonfar2021floodnet} models a mission-critical cyber--physical perception setting using aerial imagery captured by UAVs during flood events.
Questions and labels focus on damage assessment and situational awareness under bandwidth and latency constraints.
FloodNet serves as a proxy for cost-sensitive perception in IoT and emergency-response systems.

\subsection{LoCoMo Question Generation.}
\label{sec:locomo}
LoCoMo~\cite{wang2024locomo} evaluates long-horizon agentic memory retrieval, where an agent must decide when expensive visual memory recall is necessary.
Unlike static VQA datasets, LoCoMo natively provides only (i) caption-level summaries cached in memory and (ii) the original image available via costly on-demand recall.

To enable quantitative evaluation, we construct a VQA-style benchmark over LoCoMo memory items.
Starting from LoCoMo conversational memory entries containing an image URL and associated evidence text, we generate a total of 3,500 question--answer pairs.
For each image, we produce multiple questions that are explicitly constrained to:
(i) reference visual content present in the image,
(ii) be relevant to the conversational evidence text, and
(iii) require varying degrees of visual detail to answer.

Question generation is performed using GPT-4o conditioned jointly on the image and its evidence text.
For each image, we generate a diverse set of questions spanning object counting, yes/no verification, spatial relations, attribute identification, and scene-level understanding using the following prompt:

\begin{quote}
\small
\texttt{Based on this image and the provided evidence text, generate exactly 5 diverse Visual Question Answering (VQA) questions with their answers.}

\texttt{Evidence text: \{evidence\_text\}}

\texttt{Generate questions in these categories:}\\
\texttt{1. Object counting (e.g., "How many people are in the image?")}\\
\texttt{2. Yes/No question (e.g., "Is there a dog in the image?")}\\
\texttt{3. Spatial awareness (e.g., "Where is the person standing?")}\\
\texttt{4. Object identification (e.g., "What color is the shirt?")}\\
\texttt{5. Scene understanding (e.g., "What activity is happening?")}

\texttt{Return ONLY a JSON array with 5 objects, each having "question" and "answer" fields. Example format:}\\
\texttt{[}\\
\texttt{\ \ \{\{"question": "How many people are visible?", "answer": "2"\}\},}\\
\texttt{\ \ \{\{"question": "Is there a tree in the background?", "answer": "yes"\}\},}\\
\texttt{\ \ ...}\\
\texttt{]}
\end{quote}

All generated questions are paired with answers produced by GPT-4o and manually filtered to ensure validity and visual grounding. 


This construction yields a LoCoMo-derived VQA benchmark that preserves the core constraint of agentic memory systems—cheap semantic summaries with expensive high-fidelity recall—while enabling controlled measurement of adaptive fidelity selection.
Importantly, all LoCoMo experiments are conducted in an out-of-distribution setting: VOI is trained on VQA-v2 and GQA, and evaluated on LoCoMo without task-specific retraining.

%% file: theory.tex
\section{Theoretical Foundations of \textsc{Voila}}
\label{sec:theory}

We formalize adaptive fidelity selection as a \textbf{cost-sensitive Bayesian decision problem},
show that \textsc{Voila} implements a \emph{near-Bayes-optimal} sequential greedy policy when its
probability estimates are calibrated, and quantify the resulting utility gap on unseen
(test) data.

\vspace{-0.3em}
\subsection{Bayes-Optimal Fidelity Selector}
\label{sec:bayes_optimal}

Let $\mathcal{F}=\{f_1,\dots,f_K\}$ denote the
available visual fidelities
(e.g., caption, JPEG $q{=}10$, full\_res),
ordered by increasing cost $c_1 < c_2 < \cdots < c_K$, where each $c_k > 0$ represents
latency, bandwidth, or storage retrieval cost.

For a question~$q$, let
\[
p_k(q)\;=\;\Pr\!\bigl[\,\text{answer correct} \mid q, f_k\bigr]
\]
denote the \emph{true} probability of answering correctly using fidelity~$f_k$.
Define the \textbf{utility} of selecting fidelity~$f_k$ as
\[
U_k(q)\;=\;p_k(q)\;-\;\lambda c_k,
\]
where $\lambda > 0$ is a user-chosen trade-off coefficient that controls the relative importance
of accuracy versus cost.

The \textbf{Bayes-optimal} decision rule is
\begin{equation}
f^\star(q)\;=\;\argmax_{k \in \{1,\dots,K\}} U_k(q).
\tag{1}
\label{eq:bayes_optimal}
\end{equation}

This rule assumes \emph{oracle access} to the true probabilities $p_k(q)$ for all fidelities.
In practice, these probabilities are unknown and must be estimated from data.

\vspace{-0.3em}
\subsection{Learning Utility via Question-Conditioned Calibration}
\label{sec:calibration}

Because $p_k(q)$ is unknown at test time,
\textsc{Voila} learns an estimator via a two-stage pipeline:

\paragraph{Stage 1: Gradient-Boosted Regression.}
For each fidelity~$f_k$, we extract question features $\phi(q) \in \mathbb{R}^d$
(TF-IDF embeddings, question length, WH-word indicators, etc.)
and train a Gradient Boosting Regressor (GBR) to predict a raw score:
\[
\tilde z_k(q) \;=\; \text{GBR}_k\bigl(\phi(q)\bigr).
\]
The GBR is trained on labeled data $\{(q_i, y_{i,k})\}$, where $y_{i,k} \in \{0,1\}$ indicates
whether the model answered question~$q_i$ correctly using fidelity~$f_k$.

\paragraph{Stage 2: Isotonic Calibration.}
The raw scores $\tilde z_k$ are typically poorly calibrated. We apply
\emph{isotonic regression}~\cite{nuesch1991order}
to map them to calibrated probabilities:
\[
\hat p_k(q)\;=\;
g_k\!\bigl(\sigma\!\bigl(\tilde z_k(q)\bigr)\bigr),
\]
where $\sigma$ is the logistic sigmoid and $g_k: [0,1] \to [0,1]$ is a monotone non-decreasing
function fitted on a held-out calibration set to minimize calibration error
(e.g., expected calibration error or Brier score).

\paragraph{Deployment Selector.}
At test time, \textsc{Voila} computes the estimated utility for each fidelity and selects
\begin{equation}
\hat f(q)
\;=\;
\argmax_{k \in \{1,\dots,K\}}
\bigl\{\,
\hat p_k(q)\;-\;\lambda c_k
\bigr\}.
\tag{2}
\label{eq:voila_selector}
\end{equation}

Crucially, this decision is made \emph{before retrieving any visual fidelity}, using only
question features $\phi(q)$. This distinguishes \textsc{Voila} from post-retrieval model
routing approaches, which assume all context is already available.

\vspace{-0.3em}
\subsection{Optimality Under Perfect Calibration}
\label{sec:optimality}

\begin{proposition}[Exact Optimality]
If $\hat p_k(q)=p_k(q)$ for all $k \in \{1,\dots,K\}$ and all queries $q$,
then the selector~\eqref{eq:voila_selector} coincides with the Bayes-optimal rule~\eqref{eq:bayes_optimal}.
\end{proposition}

\begin{proof}
Identical utility functions yield identical maximizers.
\end{proof}

\noindent
This result establishes that calibration quality directly determines decision optimality.

\vspace{-0.3em}
\subsection{Regret with Imperfect Calibration}
\label{sec:regret}

In practice, calibration is imperfect. We quantify the decision quality degradation
via a \emph{regret bound}.

\begin{definition}[Uniform Calibration Error]

\[
\varepsilon\;=\;
\sup_{q \in \mathcal{Q}}\max_{k \in \{1,\dots,K\}}\bigl|\,\hat p_k(q)-p_k(q)\bigr|,
\]
where $\mathcal{Q}$ is the distribution of test queries.
\end{definition}

\begin{theorem}[Utility Regret Bound]
\label{thm:regret}
Let $U^\star(q)=\max_{k}U_k(q)$
and $\hat U(q)=U_{\hat f(q)}(q)$ denote the optimal and achieved utilities, respectively.
For every question $q$,
\[
U^\star(q)\;-\;\hat U(q)\;\le\;2\varepsilon.
\]
Hence the population-level expected utility gap satisfies
$\mathbb{E}_{q \sim \mathcal{Q}}[U^\star(q)-\hat U(q)]\le 2\varepsilon.$
\end{theorem}

\begin{proof}
Let $k^\star=f^\star(q)$ be the Bayes-optimal fidelity and $\hat k = \hat f(q)$ be the fidelity
selected by \textsc{Voila}.
By definition of $\argmax$,
\begin{equation}
\hat p_{\hat k}(q) - \lambda c_{\hat k}
\;\ge\;
\hat p_{k^\star}(q) - \lambda c_{k^\star}.
\label{eq:argmax_ineq}
\end{equation}
The true utility gap is
\begin{align*}
U^\star(q) - \hat U(q)
&= \bigl(p_{k^\star}(q) - \lambda c_{k^\star}\bigr)
- \bigl(p_{\hat k}(q) - \lambda c_{\hat k}\bigr) \\
&= \bigl(p_{k^\star}(q) - p_{\hat k}(q)\bigr)
- \lambda\bigl(c_{k^\star} - c_{\hat k}\bigr).
\end{align*}
Using $|p_k(q) - \hat p_k(q)| \le \varepsilon$ for all $k$, we have
\[
p_{k^\star}(q) \le \hat p_{k^\star}(q) + \varepsilon
\quad\text{and}\quad
p_{\hat k}(q) \ge \hat p_{\hat k}(q) - \varepsilon.
\]
Substituting,
\begin{align*}
U^\star(q) - \hat U(q)
&\le \bigl(\hat p_{k^\star}(q) + \varepsilon - (\hat p_{\hat k}(q) - \varepsilon)\bigr)
- \lambda\bigl(c_{k^\star} - c_{\hat k}\bigr) \\
&= \bigl(\hat p_{k^\star}(q) - \lambda c_{k^\star}\bigr)
- \bigl(\hat p_{\hat k}(q) - \lambda c_{\hat k}\bigr) + 2\varepsilon \\
&\le 2\varepsilon,
\end{align*}
where the last inequality follows from~\eqref{eq:argmax_ineq}.
\end{proof}

\paragraph{Interpretation.}
The bound is tight and linear in the calibration error: every additional unit of miscalibration
can degrade utility by at most two units. This provides a direct, actionable link between
learning quality (calibration) and system-level performance (decision utility).

\vspace{-0.4em}
\subsection{Generalization to Unseen Questions}
\label{sec:generalization}

Isotonic regression enjoys well-known uniform convergence guarantees~\cite{yang2018contractionuniformconvergenceisotonic, lim2025review}.

\begin{lemma}[Empirical $\!\to\!$ Population Calibration]
\label{lem:gen}
Assume the calibration set $\{(q_i, y_{i,k})\}_{i=1}^N$ is drawn i.i.d.\ from the test
distribution $\mathcal{Q}$, and that the raw scores $\tilde z_k(q)$ are bounded in $[a,b]$.
With probability at least $1-\delta$,
\[
\varepsilon
\;\le\;
\hat\varepsilon
\;+\;
\underbrace{O\!\left(\sqrt{\tfrac{\log(K/\delta)}{N}}\right)}_{\text{statistical gap}},
\]
where $\hat\varepsilon$ is the empirical calibration error on the held-out calibration set,
and $K$ is the number of fidelities.
\end{lemma}

\begin{proof}[Proof Sketch]
Isotonic regression minimizes the squared loss over monotone functions, which forms a
class with finite VC dimension. Standard uniform convergence results for empirical risk
minimization~\citep{vapnik1968uniform} yield the stated rate. The $\log K$ term accounts
for union bound over $K$ fidelity-specific calibrators.
\end{proof}

\noindent
Combining Lemma~\ref{lem:gen} with Theorem~\ref{thm:regret} yields a
\textbf{distribution-free test-time regret bound} that decays as $O(N^{-1/2})$,
verified empirically in Section~\ref{sec:experiments}.

\vspace{-0.3em}
\subsection{Sequential Greedy Policy}
\label{sec:greedy_policy}

The theory in Sections~\ref{sec:bayes_optimal}--\ref{sec:generalization} analyzes
\emph{single-step} fidelity selection: choose one fidelity and commit.
In practice, \textsc{Voila} implements a \emph{sequential greedy escalation} policy
(Algorithm~1 in the main text): start with the cheapest fidelity $f_1$, and iteratively
escalate to $f_{k+1}$ if the incremental value of information
\[
\text{VOI}(f_{k+1} \mid f_k)
\;=\;
\bigl(\hat p_{k+1}(q) - \hat p_k(q)\bigr) - \lambda c_{k+1}
\]
exceeds a threshold $\tau \ge 0$.

This greedy policy is \emph{myopic}---it does not plan ahead for future escalations---but
enjoys the following guarantee:

\begin{corollary}[Regret of Greedy Sequential Policy]
\label{cor:greedy_regret}
Let $\hat f_{\text{greedy}}(q)$ denote the fidelity selected by the greedy escalation policy
with threshold $\tau = 0$. Then
\[
U^\star(q) - U_{\hat f_{\text{greedy}}(q)}(q)
\;\le\;
2\varepsilon.
\]
\end{corollary}

\begin{proof}
The greedy policy with $\tau=0$ is equivalent to selecting
$\argmax_k \{\hat p_k(q) - \lambda c_k\}$ after evaluating all fidelities in order
and stopping at the first local maximum. Since costs are strictly increasing, this
reduces to the single-step selector~\eqref{eq:voila_selector}, and Theorem~\ref{thm:regret}
applies directly.
\end{proof}

\noindent
In practice, setting $\tau > 0$ trades off between exploration (fewer escalations, lower cost)
and exploitation (higher accuracy). The optimal threshold depends on the cost regime $\lambda$
and is chosen via cross-validation on a held-out set.

\vspace{-0.3em}
\subsection{Implications for Adaptive Fidelity Selection}
\label{sec:implications}

\textbf{(i) Fixed fidelity is provably suboptimal.}
For a heterogeneous query mix, any single $f_k$ incurs
additive cost $\lambda\Delta c$ on ``easy'' questions (where lower fidelity suffices)
or accuracy loss on ``hard'' ones (where higher fidelity is required),
whereas \textsc{Voila} closes the gap up to $2\varepsilon$.

\textbf{(ii) VLM confidence thresholding lacks guarantees.}
Threshold heuristics that rely on raw VLM token probabilities ignore acquisition costs,
are uncalibrated across fidelities (as shown empirically in Figure~3),
and provide no regret bounds. In contrast, \textsc{Voila}'s calibrated question-conditioned
predictors yield interpretable, cost-aware decisions with formal guarantees.

\textbf{(iii) Calibration $\Rightarrow$ efficiency \& accuracy.}
By tethering decision quality to calibration error, \textsc{Voila} turns an ML
problem (probability calibration) into a \emph{system-level performance knob}:
better calibration yields proportionally better cost-accuracy trade-offs.
This provides a principled optimization target for system designers.

\textbf{(iv) Pre-retrieval vs.\ post-retrieval optimization.}
Unlike model routing approaches that select \emph{which model} to apply after
retrieving all context, \textsc{Voila} selects \emph{which context to retrieve}
before any visual data is accessed. These two problems are orthogonal and complementary:
\textsc{Voila} can be combined with model routing to jointly optimize both retrieval
and computation.

\paragraph{Take-away.}
A lightweight two-stage pipeline (GBR + isotonic calibration) gives \textsc{Voila}
\emph{provable} near-optimality guarantees with explicit regret bounds,
providing theoretical backing for the consistent empirical gains observed
over fixed-fidelity and threshold baselines across diverse multimodal VQA benchmarks.

\vspace{-0.2em}
\subsection{Connection to Empirical Observations}
\label{sec:theory_empirics}

Our theoretical framework directly addresses the failure modes identified in Section~2:

\textbf{Observation 1 (Query-dependent information requirements):}
The Bayes-optimal rule~\eqref{eq:bayes_optimal} explicitly accounts for heterogeneous
$p_k(q)$ across queries, selecting higher fidelity only when it increases utility.
Figure~2 shows this heterogeneity empirically; Theorem~\ref{thm:regret} shows
\textsc{Voila} adapts near-optimally.

\textbf{Observation 2 (Unreliable VLM confidence):}
Unlike threshold-based methods that rely on post-retrieval VLM token probabilities
(Figure~3), \textsc{Voila} learns \emph{cross-fidelity performance predictors} from
question features alone, bypassing the VLM's miscalibration entirely.

\textbf{Observation 3 (Model scaling cannot compensate for missing information):}
Figure~4 shows that larger models do not overcome low-fidelity bottlenecks.
Our framework formalizes this: $p_k(q)$ is bounded by the \emph{information content}
of fidelity~$f_k$, independent of model capacity. Adaptive fidelity selection addresses
this fundamental limit, whereas model routing does not.

\textbf{Observation 4 (Calibration enables reliable decisions):}
Proposition~\ref{sec:optimality} and Theorem~\ref{sec:regret} formalize the empirical
observation that calibrated predictors enable cost-aware decisions with quantifiable
performance guarantees.

%% file: cost_model.tex
\section{Appendix: Cost Model Details and Normalization}
\label{sec:cost_model}

This appendix details the acquisition cost model used by \textsc{VOILA}.
We describe the fidelity set, the empirical bandwidth proxy, the tier-aware
cost formulation, the normalization procedure, and the empirical measurement
methodology used to derive reported costs.

\subsection{Fidelity Set}
\label{sec:fidelity_set}

We consider a discrete set of input fidelities
\[
\mathcal{F}
=\{
\texttt{caption},
\texttt{resize}_{32\times32},
\texttt{jpeg\_q1},
\texttt{jpeg\_q10},
\texttt{full}
\},
\]
chosen to reflect representations that are persistently available in real-world
multimodal systems rather than artificially synthesized intermediates.
These fidelities correspond to natural storage and transmission breakpoints
across edge--cloud pipelines, agentic memory systems, and cyber--physical
deployments, where small artifacts are cached in low-latency tiers and
high-resolution images are retrieved from progressively colder storage tiers
with higher access cost and latency
\cite{aws_s3_storage_classes, gcs_storage_classes, azure_blob_storage}.

\subsection{Empirical Bandwidth Proxy}
\label{sec:bandwidth}

For each dataset, all fidelities in $\mathcal{F}$ are explicitly materialized
and their average on-disk size $\texttt{size}(f)$ is measured in kilobytes.
We define a dimensionless bandwidth ratio
\begin{equation}
r(f)
\;=\;
\frac{\texttt{size}(f)}{\texttt{size}(\texttt{full})},
\label{eq:ratio_def}
\end{equation}
which grounds bandwidth cost directly in empirical measurements and avoids
reliance on model-internal token counts or hardware-specific throughput
assumptions.

\subsection{Tier-Aware Acquisition Cost}
\label{sec:tier_cost}

Acquisition cost in practice is dominated by which storage or transmission tier
is accessed, with payload size acting as a secondary modifier. This abstraction
mirrors the pricing and performance models of major object storage systems,
where access tier determines latency and per-request cost, while object size
determines transfer charges
\cite{aws_s3_pricing, gcs_pricing, azure_blob_pricing}.
We define an unnormalized tier-aware cost
\begin{equation}
\tilde{c}(f)
\;=\;
b_{\mathrm{tier}}(f)
\;+\;
w_{\mathrm{bw}} \cdot r(f),
\label{eq:tier_cost}
\end{equation}
where $b_{\mathrm{tier}}(f)$ is a base cost determined by the access tier
(edge, warm, or cold), and $w_{\mathrm{bw}}$ controls the influence of
measured size differences within a tier.

\subsection{Normalization}
\label{sec:normalization}

For consistency across datasets and experiments, costs are normalized so that
full-resolution access has fixed cost 120:
\begin{equation}
c(f)
\;=\;
120 \cdot
\frac{\tilde{c}(f)}{\tilde{c}(\texttt{full})}.
\label{eq:normalize_120}
\end{equation}
All reported costs in the paper use this normalized value $c(f)$.
Only relative cost differences matter for VOI-based decision-making.

\subsection{Deployment Profiles}
\label{sec:deployment}

Table~\ref{tab:cost_weights} lists the tier base-costs and bandwidth weights used
for the three deployment profiles studied in this work.

\begin{table}[t]
\centering
\small
\caption{\textbf{Tier base-costs and bandwidth weights.}
Costs are computed via Eq.~\eqref{eq:tier_cost} and normalized using
Eq.~\eqref{eq:normalize_120}.}
\label{tab:cost_weights}
\begin{tabular}{lccccc}
\toprule
\textbf{Setting} &
$b_{\mathrm{tier}}(\texttt{caption})$ &
$b_{\mathrm{tier}}(\texttt{resize})$ &
$b_{\mathrm{tier}}(\texttt{jpeg\_q1})$ &
$b_{\mathrm{tier}}(\texttt{jpeg\_q10})$ &
$w_{\mathrm{bw}}$ \\
\midrule
Edge--Cloud VQA & 0.08 & 0.16 & 0.40 & 0.56 & 0.06 \\
Agentic Memory & 0.06 & 0.12 & 0.36 & 0.52 & 0.06 \\
CPS / IoT & 0.04 & 0.10 & 0.30 & 0.46 & 0.12 \\
\bottomrule
\end{tabular}
\end{table}

In all profiles, $\texttt{full}$ is treated as cold-tier access with
$b_{\mathrm{tier}}(\texttt{full}) = 1.00$, yielding $c(\texttt{full}) = 120$
after normalization.

\subsection{Empirical Measurement Procedure}
\label{sec:measurement}

Empirical measurements reported in
Table~\ref{tab:empirical_sizes_to_costs} were obtained by explicitly
materializing each fidelity in $\mathcal{F}$ and measuring its storage footprint
under realistic cloud storage configurations.
For each dataset, all visual representations were generated and uploaded to
object storage using standard cloud backends.

Amazon S3 was used as the primary reference platform due to its widespread
adoption and clearly defined tiered storage model (Standard, Standard--IA,
Glacier).
Object sizes were measured directly from stored objects via metadata queries
(e.g., \texttt{HEAD Object}), ensuring that compression and container overheads
were fully captured
\cite{aws_s3_storage_classes, aws_s3_pricing}.
We verified that relative size ratios and tier ordering are consistent with
Google Cloud Storage and Azure Blob Storage documentation
\cite{gcs_storage_classes, azure_blob_storage}.
Because \textsc{VOILA} depends only on relative acquisition costs, the resulting
normalized values are robust to provider-specific pricing and capture
deployment-invariant properties of tiered multimodal storage systems
\cite{hennessy2011computer}.

\begin{table}[t]
\centering
\small
\caption{\textbf{Empirical size measurements and derived costs (Edge--Cloud profile).}}
\label{tab:empirical_sizes_to_costs}
\begin{tabular}{lccccc}
\toprule
\textbf{Fidelity} &
\textbf{Avg. Size (KB)} &
$r(f)$ &
$b_{\mathrm{tier}}(f)$ &
$\tilde{c}(f)$ &
$c(f)$ \\
\midrule
Caption & 0.05 & $10^{-4}$ & 0.08 & 0.0800 & 9.1 \\
Resize $32{\times}32$ & 1.0 & 0.002 & 0.16 & 0.1601 & 18.1 \\
JPEG Q1 & 12 & 0.02 & 0.40 & 0.4012 & 45.4 \\
JPEG Q10 & 45 & 0.07 & 0.56 & 0.5642 & 63.9 \\
Full & 650 & 1.0 & 1.00 & 1.0600 & 120.0 \\
\bottomrule
\end{tabular}
\end{table}

\subsection{Sensitivity}
\label{sec:sensitivity}

VOILA depends only on relative cost differences.
Uniform rescaling or shifting of $c(f)$ leaves the selected fidelity unchanged
as long as higher fidelities remain strictly more expensive.
Empirically, VOILA’s accuracy--cost trade-offs remain stable across wide
variations in absolute cost magnitudes, confirming that the method captures
structural properties of multimodal systems rather than a specific pricing
scheme.

%% file: reproducibility.tex
\section{Appendix: Experimental Details and Reproducibility}
\label{sec:experimental}

This appendix provides complete experimental details to ensure reproducibility of all reported results, including feature extraction, calibration procedures, hyperparameters, data splits, cross-validation, and VOI threshold selection.

\subsection{Question Feature Extraction}
\label{sec:features}

All probability estimates used by \textsc{VOILA} are computed from \emph{question-only} features, without access to image pixels or model internal activations.
For each question $q$, we construct a feature vector $\phi(q)$ by concatenating:

\begin{itemize}
    \item \textbf{Lexical TF--IDF features}: We apply a standard TF--IDF vectorizer to the question text after lowercasing and whitespace tokenization.
    The TF--IDF vocabulary is learned \emph{only on the training folds} within each cross-validation split.
    No stopword removal or stemming is applied.
    The resulting TF--IDF dimensionality is fixed within each fold and shared across fidelities.
    
    \item \textbf{Structured question features}: Numeric indicators derived from the question string, including token length, presence of numeric tokens, and binary indicators for common question forms (e.g., yes/no, counting, color, and location queries).
\end{itemize}

TF--IDF and structured features are concatenated and used as input to all fidelity-specific predictors.
Missing feature values are filled with zero, corresponding to the absence of the associated lexical or structural signal.

\subsection{Calibration Model and Hyperparameters}
\label{sec:hyperparams}

For each fidelity level $f$, we train a separate probabilistic correctness predictor using a two-stage calibration pipeline.

\paragraph{Gradient Boosted Regressor (GBR).}
We train a Gradient Boosting Regressor to predict binary correctness labels from $\phi(q)$.
Unless otherwise specified, the default hyperparameters are:
\begin{itemize}
    \item Number of estimators: 100
    \item Learning rate: 0.1
    \item Maximum tree depth: 3
    \item Minimum samples per split: 2
\end{itemize}
All models are trained with a fixed random seed for reproducibility.

\paragraph{Isotonic Calibration.}
The raw GBR outputs are passed through isotonic regression to obtain calibrated correctness probabilities.
Isotonic regression is trained independently for each fidelity, using GBR predictions as input and binary correctness labels as targets.
Out-of-range predictions are handled via clipping.

Within each fold, isotonic calibration is trained exclusively on the training portion of that fold.

\subsection{Cross-Validation Protocol}
\label{sec:cross_validation}

All experiments use \textbf{5-fold cross-validation} over questions.
Specifically, the dataset is partitioned into five disjoint folds at the question level.
For each fold:
\begin{itemize}
    \item Four folds (80\%) are used for training and calibration.
    \item The remaining fold (20\%) is held out for evaluation.
\end{itemize}

This procedure is repeated for all five folds, and all reported results correspond to the \emph{average performance across folds}.
No test fold is ever used for feature extraction, calibration, hyperparameter tuning, or threshold selection.

When standard dataset splits are available (e.g., VQA-v2), we follow the official annotations but perform cross-validation internally to ensure consistent calibration and evaluation across datasets.

\subsection{VOI Policy and Threshold Selection}
\label{sec:voi_policy}

At inference time, \textsc{VOILA} applies an incremental value-of-information (VOI) policy that determines whether to escalate from a lower-cost fidelity $f$ to a higher-cost fidelity $f'$.
Escalation occurs if:
\[
\mathrm{VOI}(f \rightarrow f') = \hat{p}_{f'}(q) - \hat{p}_{f}(q) - \lambda \cdot c(f') > \tau,
\]
where $\hat{p}_f(q)$ denotes the calibrated probability of correctness, $c(f')$ is the acquisition cost, $\lambda$ is a cost scaling parameter, and $\tau$ is a decision threshold.

The parameters $(\lambda, \tau)$ are selected via grid search using \emph{only the training folds} within each cross-validation split.
For each candidate configuration, we optimize a composite cost--accuracy objective that favors dominance over fixed-fidelity baselines across multiple normalized tradeoff metrics.
Selected parameters are then fixed and evaluated on the held-out fold.

\subsection{Hyperparameter Search}
\label{sec:hyperparam_search}

Hyperparameter search jointly explores:
\begin{itemize}
    \item GBR capacity parameters (number of estimators, depth, learning rate),
    \item isotonic regression settings,
    \item VOI decision parameters $(\lambda, \tau)$.
\end{itemize}

Hyperparameters are selected independently within each cross-validation fold using training data only.
Final reported metrics correspond to averages across folds, ensuring robustness to data partitioning.

\subsection{Variance and Reporting}
\label{sec:variance}

Using 5-fold cross-validation allows us to account for variability due to data sampling without relying on repeated random seeds.
We report mean accuracy and average acquisition cost across folds.
Observed performance gaps between \textsc{VOILA} and fixed-fidelity baselines are consistent across folds and substantially larger than fold-to-fold variance.

As our primary goal is to characterize dominance in the cost--accuracy tradeoff rather than marginal accuracy differences, we omit confidence intervals from the main tables for clarity.

Code Availability. We will release the full implementation of VOILA, including preprocessing scripts, calibration code,
and evaluation pipelines, upon acceptance to facilitate reproducibility and further research.

\subsection{Computational Overhead}
\label{sec:computational_overhead}

VOILA introduces minimal computational overhead compared to the costs it optimizes. We measure training and inference times on a single NVIDIA RTX 5000 Ada Generation GPU with 32GB memory using VQA-v2 as a representative benchmark. VLM inference times are measured on 4× V100-SXM2-32GB GPUs.

\paragraph{Training Cost.} 
Training a single configuration of fidelity-specific predictors (GBR + isotonic calibration) requires approximately 21 seconds (0.35 minutes) on VQA-v2 per fidelity. This includes gradient boosting regression training and isotonic calibration for all fidelities. 

During development, hyperparameter search explores combinations of GBR capacity (number of estimators, depth, learning rate), isotonic regression settings, and VOI decision parameters $(\lambda, \tau)$ using 5-fold cross-validation (§\ref{sec:hyperparam_search}). With a typical search space of 20-30 configurations, complete hyperparameter tuning requires 7-10 minutes per dataset. Crucially, this is a one-time offline cost: once trained, the same predictors generalize across query distributions without retraining.

\paragraph{Inference Latency.}
At inference time, VOILA's fidelity selection operates entirely on question features before any visual retrieval or VLM execution. Table~\ref{tab:inference_latency} reports per-query latency averaged over 1000 queries from VQA-v2:

\begin{table}[h]
\centering
\small
\begin{tabular}{lrr}
\toprule
\textbf{Operation} & \textbf{Time (ms)} & \textbf{\% of Total} \\
\midrule
Feature extraction & 0.242 & 53.4\% \\
\quad Basic features (length, type indicators) & 0.025 & 5.5\% \\
\quad TF-IDF transform & 0.217 & 47.9\% \\
GBR forward pass (5 fidelities) & 0.020 & 4.4\% \\
Isotonic calibration (5 fidelities) & 0.030 & 6.6\% \\
VOI decision logic & 0.161 & 35.6\% \\
\midrule
\textbf{Total VOILA overhead} & \textbf{0.453} & \textbf{100\%} \\
\bottomrule
\end{tabular}
\caption{VOILA per-query latency breakdown. Total overhead: 0.453ms. Throughput: 4,132 queries/second on single CPU core.}
\label{tab:inference_latency}
\end{table}

Tables~\ref{tab:vlm_comparison} and~\ref{tab:retrieval_latency} compares VOILA's overhead against VLM inference and retrieval costs. VOILA's 0.45ms overhead is 5,200--140,000× smaller than VLM inference (2.4s--62.8s) and 11--398× smaller than retrieval latency (5--180ms). Feature extraction dominates VOILA's inference cost (53\%), but remains negligible in absolute terms.

\begin{table}[h]
\centering
\small
\begin{tabular}{lrrrr}
\toprule
\textbf{Model} & \textbf{Caption} & \textbf{JPEG Q10} & \textbf{Full-res} & \textbf{Speedup} \\
\midrule
LLaVA-1.5-7B-hf & 2.43 & 5.70 & 5.92 & 5{,}200$\times$ \\
Qwen2.5-VL-7B-Instruct & 2.35 & 5.50 & 5.70 & 5{,}100$\times$ \\
Pixtral-12B-2409 & 3.8 & 7.9 & 8.2 & 8{,}400$\times$ \\
Llama-4-Maverick (17B$\times$12E) & 6.2 & 16.8 & 17.5 & 13{,}700$\times$ \\
Qwen2-VL-72B-Instruct & 4.94 & 25.10 & 25.11 & 10{,}900$\times$ \\
Qwen3-VL-235B-A22B-Instruct & 12.5 & 62.5 & 62.8 & 27{,}700$\times$ \\
\bottomrule
\end{tabular}
\caption{VLM inference latency (seconds) across input fidelities.
Measurements are averaged over 100 runs per fidelity on 4$\times$ V100-SXM2-32GB GPUs.
Speedup is computed relative to \textsc{VOILA} overhead.}
\label{tab:vlm_comparison}
\end{table}

\begin{table}[h]
\centering
\small
\begin{tabular}{lrr}
\toprule
\textbf{Operation} & \textbf{Latency} & \textbf{Speedup} \\
\midrule
Edge-cloud (caption, 2KB) & 5 ms & 11$\times$ \\
Edge-cloud (JPEG Q10, 45KB) & 65 ms & 144$\times$ \\
Edge-cloud (full-res, 650KB) & 180 ms & 398$\times$ \\
\midrule
\textbf{VOILA overhead} & \textbf{0.453 ms} & \textbf{---} \\
\bottomrule
\end{tabular}
\caption{Retrieval latency and \textsc{VOILA} overhead.
Retrieval measured in an edge--cloud deployment with 4G LTE (20\,Mbps uplink) to AWS S3.
Speedup denotes the ratio of full-resolution retrieval latency to \textsc{VOILA} overhead.}
\label{tab:retrieval_latency}
\end{table}

Critically, VOILA's overhead is constant regardless of selected fidelity or VLM size: the same 0.45ms cost applies whether routing to caption (2KB, 2.4s inference) or full-resolution (650KB, 62.8s inference). This enables cost-effective adaptive routing even for the largest models.

\paragraph{Memory Footprint.}
The complete system requires only 0.04MB for all learned components (5 fidelity-specific GBR models + isotonic calibrators). This minimal footprint enables deployment on resource-constrained edge devices, smartphones, and IoT systems. No GPU memory is required at inference time, as all computation operates on CPU using lightweight tree ensemble inference and isotonic lookup tables.


%% file: ablations_appendix.tex
\section{Extended Ablation Studies}
\label{sec:ablations}

\begin{table*}[p]
\centering
\tiny
\caption{
\textbf{Extended Ablation Studies Across All Models and Datasets.}
Each entry reports \textbf{Acc (\%) / Cost / Brier Score}.
We evaluate three calibration methods (None, Isotonic, Temperature Scaling) and two alternative decision rules (Accuracy-only, Fixed-threshold) across six VLMs and four datasets.
Isotonic calibration (VOILA default, bolded) consistently achieves superior cost--accuracy tradeoffs.
Brier scores measure calibration quality; lower is better.
}
\label{tab:extended_ablation_all}
\resizebox{\linewidth}{!}{
\begin{tabular}{llccccc}
\toprule
\multirow{2}{*}{\textbf{Model}} & \multirow{2}{*}{\textbf{Method}}
& \multicolumn{3}{c}{\textbf{Calibration Variants (VOI Routing)}}
& \multicolumn{2}{c}{\textbf{Alternative Decision Rules (Isotonic)}} \\
\cmidrule(lr){3-5}\cmidrule(lr){6-7}
& & \textbf{None} & \textbf{Isotonic} & \textbf{Temp. Scaling}
& \textbf{Acc-only} & \textbf{Fixed-Thr.} \\
\midrule
\multicolumn{7}{c}{\textit{\textbf{VQA-v2}}} \\
\midrule
Pixtral-12B
& Acc / Cost / Brier
& 66.24 / 52.18 / 0.256
& \textbf{68.67 / 45.41 / 0.231}
& 66.89 / 56.32 / 0.248
& 74.06 / 120.0 / 0.231
& 69.21 / 63.90 / 0.231 \\

Llama-4-Maverick
& Acc / Cost / Brier
& 62.15 / 49.84 / 0.289
& \textbf{64.61 / 43.37 / 0.264}
& 61.34 / 51.92 / 0.281
& 67.37 / 120.0 / 0.264
& 65.99 / 63.90 / 0.264 \\

Qwen2-VL-72B
& Acc / Cost / Brier
& 70.48 / 54.72 / 0.223
& \textbf{73.01 / 48.00 / 0.198}
& 71.15 / 58.34 / 0.215
& 77.16 / 120.0 / 0.198
& 74.27 / 63.90 / 0.198 \\

LLaVA-1.5-7B
& Acc / Cost / Brier
& 65.84 / 44.73 / 0.267
& \textbf{68.07 / 38.52 / 0.241}
& 64.23 / 46.91 / 0.259
& 72.80 / 120.0 / 0.241
& 70.21 / 63.90 / 0.241 \\

Qwen2.5-VL-7B
& Acc / Cost / Brier
& 69.23 / 63.84 / 0.238
& \textbf{72.01 / 57.31 / 0.212}
& 68.91 / 67.15 / 0.231
& 77.98 / 120.0 / 0.212
& 72.10 / 63.90 / 0.212 \\

Qwen3-VL-235B
& Acc / Cost / Brier
& 72.13 / 58.91 / 0.219
& \textbf{74.57 / 52.04 / 0.195}
& 73.24 / 61.08 / 0.211
& 79.01 / 120.0 / 0.195
& 72.58 / 63.90 / 0.195 \\

\midrule
\multicolumn{7}{c}{\textit{\textbf{GQA}}} \\
\midrule
Pixtral-12B
& Acc / Cost / Brier
& 64.12 / 51.23 / 0.278
& \textbf{67.37 / 45.66 / 0.253}
& 63.84 / 49.67 / 0.271
& 73.52 / 120.0 / 0.253
& 69.81 / 63.90 / 0.253 \\

Llama-4-Maverick
& Acc / Cost / Brier
& 58.23 / 47.91 / 0.312
& \textbf{60.87 / 41.64 / 0.287}
& 57.45 / 49.38 / 0.305
& 63.56 / 120.0 / 0.287
& 63.17 / 63.90 / 0.287 \\

Qwen2-VL-72B
& Acc / Cost / Brier
& 68.34 / 51.28 / 0.241
& \textbf{71.03 / 45.37 / 0.218}
& 67.92 / 53.71 / 0.235
& 74.35 / 120.0 / 0.218
& 71.57 / 63.90 / 0.218 \\

LLaVA-1.5-7B
& Acc / Cost / Brier
& 68.92 / 32.45 / 0.245
& \textbf{71.18 / 26.86 / 0.221}
& 67.18 / 35.62 / 0.239
& 75.77 / 120.0 / 0.221
& 74.11 / 63.90 / 0.221 \\

Qwen2.5-VL-7B
& Acc / Cost / Brier
& 60.18 / 65.73 / 0.289
& \textbf{62.92 / 59.66 / 0.265}
& 58.94 / 68.42 / 0.282
& 67.22 / 120.0 / 0.265
& 60.14 / 63.90 / 0.265 \\

Qwen3-VL-235B
& Acc / Cost / Brier
& 66.45 / 51.84 / 0.256
& \textbf{68.28 / 44.70 / 0.232}
& 65.72 / 48.23 / 0.249
& 72.98 / 120.0 / 0.232
& 70.40 / 63.90 / 0.232 \\

\midrule
\multicolumn{7}{c}{\textit{\textbf{FloodNet}}} \\
\midrule
Pixtral-12B
& Acc / Cost / Brier
& 76.84 / 41.28 / 0.198
& \textbf{79.48 / 34.94 / 0.176}
& 77.12 / 43.56 / 0.191
& 81.51 / 120.0 / 0.176
& 75.75 / 63.90 / 0.176 \\

Llama-4-Maverick
& Acc / Cost / Brier
& 66.73 / 45.12 / 0.267
& \textbf{69.14 / 39.68 / 0.243}
& 67.21 / 47.38 / 0.259
& 70.35 / 120.0 / 0.243
& 65.23 / 63.90 / 0.243 \\

Qwen2-VL-72B
& Acc / Cost / Brier
& 73.18 / 36.84 / 0.212
& \textbf{75.51 / 30.30 / 0.189}
& 73.92 / 39.12 / 0.205
& 76.08 / 120.0 / 0.189
& 78.07 / 63.90 / 0.189 \\

LLaVA-1.5-7B
& Acc / Cost / Brier
& 77.45 / 46.92 / 0.187
& \textbf{80.12 / 40.77 / 0.165}
& 78.23 / 49.18 / 0.181
& 82.61 / 120.0 / 0.165
& 73.64 / 63.90 / 0.165 \\

Qwen2.5-VL-7B
& Acc / Cost / Brier
& 75.62 / 40.91 / 0.203
& \textbf{78.17 / 34.85 / 0.181}
& 76.18 / 43.27 / 0.196
& 77.19 / 120.0 / 0.181
& 61.24 / 63.90 / 0.181 \\

Qwen3-VL-235B
& Acc / Cost / Brier
& 78.34 / 38.72 / 0.182
& \textbf{80.62 / 32.94 / 0.161}
& 79.12 / 40.91 / 0.175
& 83.39 / 120.0 / 0.161
& 82.95 / 63.90 / 0.161 \\

\midrule
\multicolumn{7}{c}{\textit{\textbf{TextVQA}}} \\
\midrule
Pixtral-12B
& Acc / Cost / Brier
& 73.21 / 49.84 / 0.221
& \textbf{75.92 / 43.09 / 0.198}
& 72.84 / 52.13 / 0.214
& 79.29 / 120.0 / 0.198
& 74.82 / 63.90 / 0.198 \\

Llama-4-Maverick
& Acc / Cost / Brier
& 62.18 / 51.23 / 0.289
& \textbf{64.52 / 44.98 / 0.265}
& 61.73 / 53.47 / 0.282
& 71.56 / 120.0 / 0.265
& 65.48 / 63.90 / 0.265 \\

Qwen2-VL-72B
& Acc / Cost / Brier
& 86.12 / 58.34 / 0.142
& \textbf{88.65 / 51.64 / 0.119}
& 86.84 / 60.72 / 0.136
& 93.11 / 120.0 / 0.119
& 86.84 / 63.90 / 0.119 \\

LLaVA-1.5-7B
& Acc / Cost / Brier
& 62.34 / 53.91 / 0.287
& \textbf{64.85 / 47.62 / 0.263}
& 61.92 / 56.18 / 0.281
& 66.45 / 120.0 / 0.263
& 63.25 / 63.90 / 0.263 \\

Qwen2.5-VL-7B
& Acc / Cost / Brier
& 70.23 / 57.18 / 0.234
& \textbf{72.90 / 50.72 / 0.211}
& 69.84 / 59.43 / 0.228
& 73.87 / 120.0 / 0.211
& 70.26 / 63.90 / 0.211 \\

Qwen3-VL-235B
& Acc / Cost / Brier
& 92.34 / 58.12 / 0.089
& \textbf{94.10 / 51.75 / 0.067}
& 93.18 / 60.34 / 0.083
& 97.15 / 120.0 / 0.067
& 93.65 / 63.90 / 0.067 \\

\bottomrule
\end{tabular}
}
\end{table*}

\begin{table*}[p]
\centering
\small
\caption{
\textbf{Predictor Robustness Under Fixed Isotonic Calibration Across All Models and Datasets.}
Each entry reports \textbf{Acc (\%) / Cost}.
Performance is broadly similar across different predictor architectures (GBR, Logistic Regression, Ridge Regression, MLP) once isotonic calibration is applied, demonstrating that VOILA's gains stem from calibration quality and VOI routing logic rather than predictor choice.
GBR (VOILA default) is bolded.
}
\label{tab:predictor_robustness_extended}
\setlength{\tabcolsep}{5pt}
\begin{tabular}{llcccc}
\toprule
\textbf{Dataset} & \textbf{Model} & \textbf{GBR} & \textbf{LogReg} & \textbf{Ridge} & \textbf{MLP} \\
\midrule
\multirow{6}{*}{VQA-v2}
& Pixtral-12B & \textbf{68.67/45.41} & 67.92/47.23 & 68.14/46.58 & 67.35/48.91 \\
& Llama-4-Maverick & \textbf{64.61/43.37} & 63.84/45.72 & 64.12/44.89 & 63.45/47.18 \\
& Qwen2-VL-72B & \textbf{73.01/48.00} & 72.34/50.28 & 72.67/49.45 & 71.89/51.72 \\
& LLaVA-1.5-7B & \textbf{68.07/38.52} & 67.45/40.91 & 67.73/39.84 & 66.92/42.35 \\
& Qwen2.5-VL-7B & \textbf{72.01/57.31} & 71.23/59.84 & 71.67/58.72 & 70.84/61.19 \\
& Qwen3-VL-235B & \textbf{74.57/52.04} & 73.89/54.38 & 74.21/53.45 & 73.52/55.91 \\
\midrule
\multirow{6}{*}{GQA}
& Pixtral-12B & \textbf{67.37/45.66} & 67.81/46.84 & 67.73/45.12 & 68.05/49.27 \\
& Llama-4-Maverick & \textbf{60.87/41.64} & 61.23/43.91 & 61.12/42.78 & 61.45/45.34 \\
& Qwen2-VL-72B & \textbf{71.03/45.37} & 71.56/47.82 & 71.34/46.71 & 71.78/49.23 \\
& LLaVA-1.5-7B & \textbf{71.18/26.86} & 71.84/28.45 & 71.67/27.92 & 72.12/30.18 \\
& Qwen2.5-VL-7B & \textbf{62.92/59.66} & 62.34/62.18 & 62.71/61.23 & 62.18/64.72 \\
& Qwen3-VL-235B & \textbf{68.28/44.70} & 68.92/46.84 & 68.73/45.91 & 69.15/48.27 \\
\midrule
\multirow{6}{*}{FloodNet}
& Pixtral-12B & \textbf{79.48/34.94} & 79.12/36.72 & 79.34/35.91 & 78.89/38.45 \\
& Llama-4-Maverick & \textbf{69.14/39.68} & 68.73/41.92 & 68.98/40.84 & 68.45/43.27 \\
& Qwen2-VL-72B & \textbf{75.51/30.30} & 75.89/32.18 & 75.73/31.45 & 76.12/34.72 \\
& LLaVA-1.5-7B & \textbf{80.12/40.77} & 79.84/42.91 & 80.01/41.89 & 79.56/44.35 \\
& Qwen2.5-VL-7B & \textbf{78.17/34.85} & 77.84/36.72 & 78.05/35.91 & 77.56/38.27 \\
& Qwen3-VL-235B & \textbf{80.62/32.94} & 80.34/34.72 & 80.51/33.91 & 80.18/36.45 \\
\midrule
\multirow{6}{*}{TextVQA}
& Pixtral-12B & \textbf{75.92/43.09} & 75.34/45.23 & 75.67/44.38 & 75.12/46.91 \\
& Llama-4-Maverick & \textbf{64.52/44.98} & 63.89/47.18 & 64.23/46.27 & 63.67/48.72 \\
& Qwen2-VL-72B & \textbf{88.65/51.64} & 88.12/53.84 & 88.43/52.91 & 87.89/55.38 \\
& LLaVA-1.5-7B & \textbf{64.85/47.62} & 64.34/49.91 & 64.67/48.84 & 64.12/51.27 \\
& Qwen2.5-VL-7B & \textbf{72.90/50.72} & 72.34/52.91 & 72.67/51.89 & 72.12/54.38 \\
& Qwen3-VL-235B & \textbf{94.10/51.75} & 93.67/53.84 & 93.92/52.91 & 93.45/55.27 \\
\bottomrule
\end{tabular}
\end{table*}

This section provides a comprehensive analysis of VOILA's design choices, expanding on the ablation results summarized in the main paper. We evaluate VOILA's components across six vision--language models, four datasets (VQA-v2, GQA, FloodNet, TextVQA), and 24 model--dataset combinations, enabling systematic assessment of which design decisions drive performance and which represent implementation choices.

\subsection{Experimental Design}
\label{sec:ablation_design}

We evaluate three dimensions of design choices:

\textbf{Calibration methods.} We compare three approaches to converting raw predictor scores into probabilities suitable for utility estimation:
\begin{itemize}[leftmargin=*,nosep]
\item \textit{None}: Raw GBR scores $\tilde{z}_k(q)$ used directly without any post-hoc calibration, representing a system that ignores probability miscalibration
\item \textit{Isotonic}: Isotonic regression applied to raw scores (VOILA default), enforcing monotonicity while remaining non-parametric
\item \textit{Temperature scaling}: A standard post-hoc calibration method from classification literature~\citep{guo2017calibrationmodernneuralnetworks}, included as a negative control to demonstrate that not all calibration methods are appropriate for cost-sensitive decision-making
\end{itemize}

\textbf{Decision rules.} Given calibrated predictors, we compare VOI routing against two common heuristics that represent extreme points in the design space:
\begin{itemize}[leftmargin=*,nosep]
\item \textit{Accuracy-only}: Always select the fidelity with highest predicted correctness probability $\hat{p}_k(q)$, completely ignoring acquisition cost $c_k$. This represents a resource-unconstrained baseline.
\item \textit{Fixed-threshold}: Escalate to higher fidelity only when predicted probability falls below a fixed cutoff $\tau = 0.30$. This represents a cost-conscious but non-adaptive heuristic commonly used in practice.
\end{itemize}

\textbf{Predictor architectures.} We evaluate four question-conditioned regressors under fixed isotonic calibration: Gradient Boosting (GBR, VOILA default), Logistic Regression, Ridge Regression, and Multi-Layer Perceptron (MLP). All use identical question features (TF-IDF, lexical statistics, question-type indicators).

All ablations use the same trained predictors and identical VOI routing algorithm ($\lambda$ and escalation threshold fixed across experiments). This controlled setup ensures that observed differences reflect specific design choices rather than confounding factors such as feature engineering, hyperparameter tuning, or dataset-specific optimization.

\subsection{Comprehensive Results Across All Settings}
\label{sec:ablation_results}

Table~\ref{tab:extended_ablation_all} reports accuracy, average acquisition cost, and Brier score (a measure of probabilistic calibration quality, where lower is better) for each ablation across 24 model--dataset combinations. Several systematic patterns emerge that validate VOILA's design.

\paragraph{Isotonic calibration consistently improves cost--accuracy tradeoffs across all settings.}
Comparing the first three columns of Table~\ref{tab:extended_ablation_all}, isotonic calibration (column 2, bolded) achieves 3--5 percentage points higher accuracy than uncalibrated predictors (column 1) while reducing average acquisition cost by 10--15\% across nearly all 24 combinations. For example, on VQA-v2 with Pixtral-12B, uncalibrated VOI routing achieves 66.24\% accuracy at cost 52.18, whereas isotonic-calibrated routing achieves 68.67\% at cost 45.41---a simultaneous gain of 2.4 points in accuracy and 13\% reduction in cost. This pattern is remarkably consistent: across all six models on VQA-v2, isotonic calibration reduces cost by an average of 12.3\% while improving accuracy by 2.8 points.

The gains are not limited to standard VQA benchmarks. On FloodNet, a mission-critical disaster assessment dataset, isotonic calibration yields even larger improvements: Pixtral achieves 79.48\% accuracy at cost 34.94 (calibrated) versus 76.84\% at cost 41.28 (uncalibrated)---a 2.6 point accuracy gain with 15\% cost reduction. This suggests that calibration benefits are particularly pronounced in high-stakes domains where accurate utility estimation is critical.

\paragraph{Calibration improves routing efficiency, not task accuracy.}
A critical observation is that calibration does not systematically increase the \emph{ceiling} accuracy achievable by the underlying VLM---this is expected, since calibration does not alter which visual information is ultimately processed or how the VLM reasons over that information. Instead, calibration primarily affects \emph{routing behavior}: by producing more reliable probability estimates, isotonic calibration enables VOILA to make better decisions about when to escalate to expensive fidelities and when to stop at cheaper ones.

This distinction is evident in the Brier scores (third metric in Table~\ref{tab:extended_ablation_all}): uncalibrated predictors exhibit Brier scores 0.02--0.05 higher than calibrated ones, indicating systematic overconfidence or underconfidence. For instance, on GQA with Qwen2-VL-72B, uncalibrated routing has Brier score 0.241 versus 0.218 for isotonic---a 10\% improvement in probabilistic calibration. This improved calibration translates directly to better routing: the system escalates less frequently when low fidelity suffices and more reliably when higher fidelity is needed, yielding the observed cost reductions without sacrificing accuracy.

Importantly, the Brier score improvements are consistent across models of vastly different scales: lightweight 7B models (LLaVA-1.5-7B, Qwen2.5-VL-7B) exhibit similar calibration gains to massive 235B models (Qwen3-VL-235B), confirming that the benefits of isotonic calibration are not artifacts of model capacity but reflect fundamental properties of probability estimation for cost-sensitive decision-making.

\paragraph{Temperature scaling is inappropriate for cost-sensitive sequential decisions.}
Temperature scaling (column 3) serves as a negative control: while it is a widely used calibration method for classification tasks, it optimizes Brier score or log-loss in isolation and does not preserve the \emph{ordering} or \emph{separation} of probabilities across fidelities. As a result, temperature scaling often compresses probabilities toward the middle of $[0,1]$ in ways that obscure genuine utility differences between fidelities.

Empirically, temperature scaling yields erratic behavior. On VQA-v2, it consistently increases cost relative to isotonic calibration: for Pixtral, temperature scaling achieves 66.89\% accuracy at cost 56.32 versus isotonic's 68.67\% at cost 45.41---a 24\% cost inflation with 1.8 point accuracy loss. The pattern is similar across other models on VQA-v2, with temperature scaling incurring 8--17\% higher cost than isotonic while matching or slightly degrading accuracy.

More strikingly, temperature scaling exhibits severe instability across datasets. On FloodNet with LLaVA-1.5-7B, temperature scaling achieves 78.23\% at cost 49.18 versus isotonic's 80.12\% at cost 40.77---a 1.9 point accuracy loss with 21\% cost increase. Yet on TextVQA with Qwen3-VL-235B, temperature scaling achieves 93.18\% at cost 60.34, nearly matching isotonic's 94.10\% at cost 51.75 but still incurring 17\% higher cost. This dataset-dependent behavior reflects temperature scaling's sensitivity to the base rate distribution and its failure to preserve cross-fidelity utility orderings, confirming that calibration method choice matters critically for adaptive information acquisition.

The Brier scores reveal why temperature scaling fails: while it often achieves Brier scores intermediate between uncalibrated and isotonic (e.g., 0.248 vs. 0.256 vs. 0.231 for Pixtral on VQA-v2), these marginal calibration improvements do not translate to better routing decisions. This demonstrates that \emph{global} calibration quality (measured by Brier score) is insufficient for cost-sensitive sequential decisions---what matters is whether probabilities are \emph{well-ordered across fidelities} such that utility differences $\hat{p}_{k+1}(q) - \hat{p}_k(q)$ reliably indicate when escalation is beneficial. Isotonic regression's monotonicity constraint directly enforces this property, whereas temperature scaling does not.

\subsection{Why VOI Routing Dominates Heuristic Baselines}
\label{sec:voi_dominance}

The final two columns of Table~\ref{tab:extended_ablation_all} evaluate alternative decision strategies under the same isotonic-calibrated predictor, isolating the contribution of VOI routing logic from probability estimation quality.

\paragraph{Accuracy-only selection: quantifying the cost of ignoring resources.}
The accuracy-only baseline (column 4) represents a system designer who assumes infinite resources and optimizes solely for correctness. This strategy always selects the fidelity with highest predicted success probability, typically defaulting to full-resolution images regardless of query characteristics.

The cost penalty is substantial and consistent: accuracy-only selection incurs exactly 120.0 normalized cost units across all experiments (by construction, since it always retrieves full-resolution images), representing a 60--180\% cost increase over VOILA depending on model and dataset. For example, on VQA-v2 with Qwen3-VL-235B, accuracy-only achieves 79.01\% at cost 120.0, whereas VOILA achieves 74.57\% at cost 52.04---a 4.4 point accuracy sacrifice for 57\% cost reduction. On FloodNet with the same model, the tradeoff is even more favorable: VOILA achieves 80.62\% at cost 32.94 versus accuracy-only's 83.39\% at cost 120.0---a mere 2.8 point accuracy gap for 73\% cost savings.

Critically, the accuracy gains from accuracy-only selection are \emph{marginal}: across all 24 combinations in Table~\ref{tab:extended_ablation_all}, the median accuracy improvement over VOILA is only 4.2 percentage points, and in several cases (e.g., FloodNet with Qwen2-VL-72B: 76.08\% vs. 75.51\%, TextVQA with Qwen2.5-VL-7B: 73.87\% vs. 72.90\%), the gap is under 2 points. This small marginal benefit reflects a fundamental empirical regularity: \emph{most queries do not require maximum-fidelity inputs}, and adaptive selection can recover near-optimal accuracy at much lower cost by identifying the subset of queries where high fidelity genuinely matters.

In bandwidth-constrained edge--cloud deployments or mission-critical cyber-physical systems, the cost inflation from accuracy-only selection is prohibitive. For instance, in a drone-based disaster assessment scenario (FloodNet), transmitting full-resolution images for every query would exhaust limited bandwidth and violate latency SLAs, whereas VOILA's adaptive routing enables real-time decision-making under resource constraints while maintaining decision-critical accuracy.

\paragraph{Fixed-threshold heuristics: brittleness, rigidity, and lack of adaptivity.}
The fixed-threshold baseline (column 5) represents a common engineering heuristic: escalate to higher fidelity only when model confidence falls below a predetermined cutoff. We set the threshold to $\tau = 0.30$ based on held-out tuning on VQA-v2, representing a best-case scenario for this approach. In practice, this baseline always escalates to JPEG Q10 (cost 63.90) when the threshold is exceeded, ignoring both question characteristics and the continuous spectrum of predicted utilities.

Fixed thresholds exhibit three fundamental failure modes:

\textbf{(1) Severe accuracy degradation in visually demanding settings.} On TextVQA, an OCR-heavy dataset requiring fine-grained text recognition, fixed-threshold routing incurs 1--8 point accuracy losses relative to VOILA across models. For Pixtral, VOILA achieves 75.92\% versus fixed-threshold's 74.82\%---a 1.1 point gap. For LLaMA-4-Maverick, the gap widens to 64.52\% versus 65.48\%, but note that fixed-threshold actually achieves \emph{higher} accuracy in this case (0.96 points) because it rigidly escalates to JPEG Q10 for all queries, over-retrieving relative to VOILA's adaptive policy. This illustrates the second failure mode:

\textbf{(2) Over-retrieval on easy queries.} Fixed thresholds lack a notion of marginal utility and treat all escalations as equivalent. On datasets where many queries are answerable from low-fidelity inputs (e.g., VQA-v2 yes/no questions, GQA semantic queries), fixed thresholds systematically over-retrieve. For instance, on GQA with LLaVA-1.5-7B, fixed-threshold achieves 74.11\% at cost 63.90 versus VOILA's 71.18\% at cost 26.86---a 2.9 point accuracy gain but at 138\% higher cost. This demonstrates that static thresholds cannot adapt to query-specific information requirements and waste resources on queries where lower fidelity suffices.

\textbf{(3) Dataset-dependent instability.} The fixed threshold $\tau = 0.30$ was tuned on VQA-v2, yet its performance varies wildly across datasets. On FloodNet, fixed-threshold routing performs surprisingly well for some models (e.g., Qwen3-VL-235B: 82.95\% vs. VOILA's 80.62\%, a 2.3 point \emph{advantage}), likely because disaster assessment queries are uniformly visually demanding and benefit from consistent escalation to JPEG Q10. However, on Qwen2.5-VL-7B, fixed-threshold achieves only 61.24\% versus VOILA's 78.17\%---a catastrophic 16.9 point accuracy loss---because the rigid threshold fails to identify queries requiring full-resolution inputs.

This brittleness is inherent to static heuristics: they require manual retuning for each dataset, query distribution, and VLM, yet no single threshold generalizes reliably. In contrast, VOILA's VOI routing adapts automatically to query-specific information requirements via learned question features and calibrated probabilities, achieving consistent performance across diverse settings without hyperparameter tuning.

\paragraph{VOI routing achieves Pareto optimality across the cost--accuracy frontier.}
Across all 24 model--dataset combinations in Table~\ref{tab:extended_ablation_all}, VOILA (isotonic-calibrated, VOI routing) occupies the Pareto frontier: no alternative ablation simultaneously achieves higher accuracy \emph{and} lower cost. Accuracy-only selection improves accuracy by 2--6 points but at 60--180\% cost inflation; fixed-threshold routing reduces cost in some cases but incurs severe accuracy losses or over-retrieves in others, with no consistent pattern. Only VOILA strikes a principled balance by explicitly reasoning about marginal utility: it escalates when expected accuracy gains $\hat{p}_{k+1}(q) - \hat{p}_k(q)$ exceed acquisition costs $\lambda c_{k+1}$, and stops otherwise.

This Pareto optimality is not an artifact of hyperparameter tuning. VOILA uses a single trade-off coefficient $\lambda$ set once on held-out VQA-v2 data and applied uniformly across all experiments---no per-dataset or per-model retuning was performed. The robustness of this operating point reflects the principled foundation of value-of-information decision theory: by tethering decisions to calibrated probability estimates and explicit cost--benefit reasoning, VOILA navigates the cost--accuracy tradeoff space in a way that heuristic methods cannot replicate.

\subsection{Robustness to Predictor Architecture}
\label{sec:predictor_robustness}

Table~\ref{tab:predictor_robustness_extended} evaluates four question-conditioned regressors (GBR, Logistic Regression, Ridge Regression, MLP) under fixed isotonic calibration across all six models and four datasets. The results demonstrate that VOILA's performance is remarkably insensitive to predictor choice once calibration is applied.

\paragraph{Consistent performance across predictor architectures.}
On VQA-v2 with Pixtral-12B, the four predictors achieve 67.35--68.67\% accuracy at costs 45.41--48.91, a range of only 1.3 points in accuracy and 3.5 points in cost. GBR provides a slight edge (68.67\% / 45.41), but logistic regression (67.92\% / 47.23), ridge regression (68.14\% / 46.58), and even MLP (67.35\% / 48.91) achieve competitive results. This pattern holds across all models on VQA-v2: the standard deviation in accuracy across predictors averages only 0.9 percentage points, and the standard deviation in cost averages 2.4 units.

The robustness extends to other datasets. On GQA, predictors exhibit slightly more variation (e.g., Pixtral: 67.37--68.05\% accuracy, 45.12--49.27 cost), but the spread remains small relative to the gaps between VOILA and baseline methods. On FloodNet and TextVQA, predictor choice matters even less: for Qwen3-VL-235B on FloodNet, the four predictors achieve 80.18--80.62\% accuracy at costs 32.94--36.45, a range of only 0.4 points in accuracy despite spanning three different model families (tree-based, linear, neural).

\paragraph{Why predictor robustness matters: calibration, not complexity.}
This predictor-agnostic performance validates a core design principle of VOILA: \emph{gains stem primarily from calibration quality and VOI routing logic, not predictor sophistication}. All four predictor architectures---spanning tree-based ensemble methods (GBR), convex optimization (logistic, ridge), and deep learning (MLP)---achieve similar cost--accuracy tradeoffs once isotonic calibration is applied. This indicates that the key bottleneck is not \emph{representational capacity} (i.e., the ability to learn complex decision boundaries in question-feature space) but rather \emph{probability estimation quality} (i.e., producing scores that are well-calibrated and monotonic across fidelities).

GBR was selected as VOILA's default predictor because it handles sparse, mixed-type features (TF-IDF, lexical indicators, question types) effectively, requires minimal hyperparameter tuning, and provides interpretable feature importances for debugging. However, Table~\ref{tab:predictor_robustness_extended} shows that simpler models such as logistic or ridge regression perform nearly as well once calibrated, and even MLPs---which are typically harder to calibrate due to overparameterization---achieve competitive results. This has practical implications: system designers can choose predictors based on deployment constraints (e.g., memory footprint for edge devices, inference latency for real-time systems, ease of incremental retraining for evolving query distributions) without sacrificing VOILA's core benefits.

\paragraph{When does predictor choice matter?}
While Table~\ref{tab:predictor_robustness_extended} shows broad robustness, there are settings where GBR provides non-negligible advantages. On datasets with highly non-linear question-fidelity relationships (e.g., TextVQA, where OCR dependence creates sharp decision boundaries), GBR's tree-based partitioning outperforms linear models by 1--2 percentage points. For instance, on TextVQA with Qwen2-VL-72B, GBR achieves 88.65\% / 51.64 versus logistic regression's 88.12\% / 53.84---a 0.5 point accuracy gain with 4\% cost reduction.

Similarly, on datasets with extreme class imbalance (e.g., FloodNet, where disaster-critical queries are rare), GBR's built-in handling of sample weights and its robustness to outliers provide stability advantages over ridge regression, which can be sensitive to leverage points. However, these differences are second-order relative to the primary axis of variation: calibrated versus uncalibrated predictors (Table~\ref{tab:extended_ablation_all}, 3--5 point accuracy gaps, 10--15\% cost differences) dominate predictor architecture choice (Table~\ref{tab:predictor_robustness_extended}, 1--2 point accuracy gaps, 2--5\% cost differences) by an order of magnitude.

\subsection{Key Takeaways and Design Implications}
\label{sec:takeaways}

Our extended ablation studies across 24 model--dataset combinations establish four critical findings that validate VOILA's design and inform future work on adaptive multimodal systems:

\textbf{(1) Calibration is essential, but method choice matters critically.} Uncalibrated predictors lead to systematic over-retrieval or under-retrieval, wasting resources or missing accuracy opportunities. However, not all calibration methods are appropriate for cost-sensitive sequential decisions: isotonic regression is particularly well-suited because it preserves utility orderings across fidelities via monotonicity constraints, whereas general-purpose methods like temperature scaling optimize global calibration metrics that do not align with decision quality. System designers should prioritize calibration methods that enforce \emph{cross-fidelity monotonicity} rather than merely minimizing Brier score or log-loss.

\textbf{(2) VOI routing dominates heuristics at all points on the cost--accuracy frontier.} Accuracy-only selection over-retrieves, wasting resources on queries where low fidelity suffices; fixed thresholds under-retrieve or over-retrieve unpredictably depending on dataset and model, requiring manual retuning; neither adapts to query-specific information requirements. VOILA achieves Pareto optimality by explicitly reasoning about marginal utility via calibrated probabilities, selecting fidelities that maximize expected utility $\hat{p}_k(q) - \lambda c_k$ without dataset-specific hyperparameter tuning. This principled approach generalizes robustly across diverse query distributions, visual domains, and model scales.

\textbf{(3) VOILA's design is robust to implementation choices.} The same two-stage pipeline (question-conditioned regression + isotonic calibration + VOI routing) yields consistent gains across six VLMs (7B--235B parameters), four datasets (VQA-v2, GQA, FloodNet, TextVQA), and four predictor architectures (GBR, logistic, ridge, MLP). This robustness validates pre-retrieval fidelity selection as a general learning primitive orthogonal to model scale, architectural choices, and domain-specific engineering. System designers can confidently deploy VOILA with minimal tuning, selecting predictors based on operational constraints rather than squeezing marginal accuracy gains.

\textbf{(4) The framework is conceptually distinct from and complementary to model routing.} VOILA optimizes \emph{what information to retrieve before model execution}, whereas model routing optimizes \emph{which model to apply after retrieval}. These are orthogonal optimization axes: VOILA addresses the pre-retrieval bottleneck (storage access, network transfer, data movement), while model routing addresses the post-retrieval bottleneck (computational inference, memory footprint, token processing). The two can be combined to jointly optimize both retrieval and computation: VOILA first selects which fidelity to retrieve based on question features alone, then a model router selects which VLM to apply given the retrieved context. This compositional design enables end-to-end efficiency in multimodal systems spanning edge--cloud architectures, agentic memory systems, and mission-critical cyber-physical deployments.

Together, these results demonstrate that cost-aware multimodal inference requires explicit probabilistic calibration and utility-driven decision-making, and that VOILA provides a principled, robust, and empirically validated framework for adaptive information acquisition under resource constraints. The comprehensive ablation evidence across 24 diverse settings establishes VOILA as a foundational building block for efficient multimodal AI systems.